\documentclass[12pt,a4paper,reqno]{amsart}
\usepackage{fullpage}
\usepackage{mathpazo}
\usepackage[english]{babel}

\usepackage[utf8]{inputenc} % allow utf-8 input
\usepackage[T1]{fontenc}    % use 8-bit T1 fonts
\usepackage{hyperref}       % hyperlinks
\usepackage{url}            % simple URL typesetting
\usepackage{booktabs}       % professional-quality tables
\usepackage{amsfonts}       % blackboard math symbols
\usepackage{nicefrac}       % compact symbols for 1/2, etc.
\usepackage{microtype} % microtypography

\usepackage{csquotes}

\usepackage{amsmath,amssymb,amsthm,bbm}

\usepackage{xy}
\input xy \xyoption{all}

\usepackage{makecell}

\usepackage{cleveref}

\usepackage[dvipsnames]{xcolor}

\usepackage{tikz}    
\usepackage{pgfplots}
\usepackage{mathrsfs}  

\usepackage{enumerate}\usepackage[inline]{enumitem}

\newcommand{\ra}[1]{\renewcommand{\arraystretch}{#1}}
\usepackage{caption}

\usepackage[
 backend=biber,
 style=numeric,
 giveninits=true,
 maxcitenames = 10,
 maxbibnames = 10
 ]{biblatex}
 \DeclareSourcemap{
  \maps[datatype=bibtex]{
    \map{
      \step[fieldset=isbn, null]
      \step[fieldset=issn, null]
      \step[fieldset=doi, null]
      \step[fieldset=url, null]
      \step[fieldset=mrclass, null]
      \step[fieldset=mrnumber, null]
      \step[fieldset=note, null]
      \step[fieldset=abstract, null]
    }
    \map[overwrite=true]{
      \step[fieldsource=fjournal]
      \step[fieldset=journal, origfieldval]
    }
  }
}
\newcommand{\be}{\beta}
\newcommand{\Th}{\Theta}
\let\th\relax\newcommand{\th}{\theta}

\newcommand{\La}{\Lambda}
\newcommand{\si}{\sigma}

\newcommand{\ga}{\gamma}

\newcommand{\al}{\alpha}
\newcommand{\de}{\delta}

\newcommand{\N}{\mathbb{N}}

\newcommand{\R}{\mathbb{R}}
\newcommand{\C}{\mathbb{C}}

\let\P\relax \newcommand{\P}{\mathbb{P}}

\let\aa\relax \newcommand{\aa}{\mathcal{A}}

\newcommand{\rr}{\mathcal{R}}
\newcommand{\xx}{\mathcal{X}}

\newcommand{\mm}{\mathcal{M}}
\newcommand{\nn}{\mathcal{N}}
\newcommand{\hh}{\mathcal{H}}
\newcommand{\bb}{\mathcal{B}}
\newcommand{\ff}{\mathcal{F}}

\newcommand{\pp}{\mathcal{P}}
\newcommand{\cS}{\mathcal{S}}
\newcommand{\cF}{\mathcal{F}}

\newcommand{\D}{{\rm d}}

\newcommand{\scal}[4]{{}_{#3}\langle #1,#2\rangle_{#4}}
\newcommand{\nor}[2]{\| #1\|_{#2}}

\DeclareMathOperator{\sgn}{sgn}

\DeclareMathOperator{\TV}{TV}

\DeclareMathOperator{\Ext}{Ext}

\theoremstyle{definition}
\newtheorem{dfn}{Definition}[section]

\theoremstyle{remark}
\newtheorem{ex}[dfn]{Example}

\theoremstyle{remark}
\newtheorem{rmk}[dfn]{Remark}
\theoremstyle{plain}
\newtheorem{lem}[dfn]{Lemma}
\newtheorem{prop}[dfn]{Proposition}
\newtheorem{thm}[dfn]{Theorem}
\newtheorem{cor}[dfn]{Corollary}

\begin{document}
\title{\bf Understanding neural networks \\ with reproducing kernel Banach spaces}
%or\\
%Neural network representer theorems \\ by Banach feature maps}

\author{F.~Bartolucci}
\address[F.~Bartolucci]{SAM - Department of Mathematics, ETH Z\"urich
%, R\"amistrasse 101, 8092 Z\"urich
, Switzerland
}
\email{francesca.bartolucci@sam.math.ethz.ch}

\author{E.~De Vito}
\address[E. De Vito]{DIMA - MaLGa, Universit\`a di Genova,
%Via Dodecaneso 35, 16146 Genova,
Italy
}
\email{ernesto.devito@unige.it}

\author{L.~Rosasco}
\address[L. Rosasco]{DIBRIS - MaLGa, Universit\`a di Genova,
%, Via Dodecaneso 35, 16146 Genova,
Italy
\& CBMM, MIT \& IIT}
\email{lorenzo.rosasco@unige.it}

\author{S.~Vigogna}
\address[S. Vigogna]{DIBRIS - MaLGa, Universit\`a di Genova,
%Via Dodecaneso 35, 16146 Genova,
Italy
}
\email{vigogna@dibris.unige.it}

\maketitle

\begin{abstract}
Characterizing the function spaces corresponding to  neural networks can provide a way to understand their properties.
In this paper we discuss how the theory of reproducing kernel Banach spaces can be used to tackle this challenge.
In particular, we prove a representer theorem for a wide class of  reproducing kernel Banach spaces that admit a suitable integral representation 
and include one hidden layer neural networks of possibly infinite width. 
Further,  we show that, for a suitable class of ReLU activation functions, the norm in the corresponding reproducing kernel Banach space can be 
 characterized in terms of the inverse Radon transform of a bounded real measure, with norm given by the total variation norm of the measure. 
 Our analysis simplifies and extends recent results in \cite{savarese2019infinite,ongie2019function,parhi2021banach}.
\end{abstract}


% \smallskip
% \noindent \textbf{AMS.} 

\section{Introduction}

Neural networks provide a flexible and effective class of machine learning models, by recursively composing  linear and nonlinear functions. 
The models thus obtained correspond to nonlinearly parameterized functions, and typically require non convex optimization procedures \cite{Goodfellow-et-al-2016}. 
While this does not prevent good empirical performances, it makes understanding neural network properties considerably  complex. 
Indeed,  characterizing what function classes can be well represented/approximated by neural networks is a classic problem \cite{pinkus_1999,JMLR:v18:14-546,savarese2019infinite,ongie2019function,parhi2021banach,gribonval2021approximation}, but it is still not fully understood. Moreover, networks with large numbers of parameters are often practically successful, seemingly contradicting the idea that models should be simple to be learned from data \cite{benunderstanding,belkin2019reconciling}.
This observation raises the question of in what sense the complexity of the models  is explicitly or implicitly controlled. From a functional perspective, the answer corresponds to understanding what norms can be defined and controlled on the spaces of functions defined by neural networks. 

Among neural networks, there is one model where the above questions become considerably more  amenable to study, namely   neural networks with only one hidden layer.
%\textcolor{blue}{\st{with infinite width}}.
In this case,
% \textcolor{blue}{taking the limit to infinite width}, 
 functions can be seen to be parameterized by measures, with networks with finitely many hidden units corresponding to  atomic measures \cite{JMLR:v18:14-546}. The remarkable advantage of this framework is that the parameterization in terms of measures is linear, and functional calculus considerably simplifies. This observation is at the base of the connection between neural networks and Gaussian processes \cite{neal2012bayesian}, as well as  random features \cite{NIPS2007_013a006f}, which allows to bring to bear the powerful machinery of reproducing kernel Hilbert spaces \cite{MR51437}.
However, starting at least from \cite{barron1993universal,barron1994approximation}, it is clear that norms other than Hilbertian can be defined that might better capture the inductive biases induced by neural networks. For example, for   functions parameterized by absolutely continuous measures, the $L^1$ norm of the corresponding densities can be considered. More generally, functional norms can be defined in terms of total variations of the corresponding measures. The study in \cite{JMLR:v18:14-546} provides a clear  discussion on this perspective. 

The extension from a Hilbert to a Banach setting opens a number of questions. We discuss two that are relevant to our study. The first one is related to the characterization of the solution of empirical minimization problems, the so-called representer theorem.
It is well known that, in a Hilbert setting, minimizers always lie in a finite dimensional subspace. Each solution is a linear combination of the reproducing kernel associated to the Hilbert space evaluated at the training set points \cite{kimeldorf1970correspondence, kimeldorf1971some, scholkopf2001generalized}. This result has immediate computational implications and is at the base of kernel methods \cite{scholkopf2002learning}. A natural question is then how these results extend to a Banach space of functions  defined by neural networks. A number of recent results tackles this question \cite{unser2020unifying,parhi2021banach}. A main difficulty is that the Banach spaces defined by neural networks are non-reflexive, and their definition requires some care. In this context, our first contribution is that we systematically use the machinery of reproducing kernel Banach spaces \cite{JMLR:v10:zhang09b,Lin2019OnRK} to simplify and analyze the construction of such spaces. In the Hilbert setting, feature maps and positive definite kernels can both be  equivalently used to define functions spaces with the reproducing property. For  non-reflexive Banach spaces,  only feature maps provide a natural approach. While a reproducing kernel can be defined, it is typically neither symmetric nor positive definite. Instead, we show that,  introducing appropriate feature maps, function spaces defined by neural networks can be seen to define reproducing kernel Banach spaces of functions admitting a suitable integral representation. Through this characterization and the application of a recent technical result in  \cite{MR4040623}, we can immediately derive a representer theorem. This result can be contrasted to \cite{parhi2021banach}, and, as discussed later, allows  dealing more directly with some technical issues.
We note in passing that representer theorems for neural networks have different implications than analogous  results in the Hilbert setting. Unlike the Hilbert setting, they do not have immediate computational consequences, 
{but have interesting  implications from a conceptual point of view. Indeed, they imply that finite networks suffice to solve empirical minimization problems. Further, they imply an upper bound on the amount of  the amount of  overparameterization required. }
%Indeed, they show that, even if we had access to infinite wide neural networks, a finite number of units suffices to solve empirical risk minimization problems. Further, they show that a number of units at most of the  cardinality of the data also suffices, suggesting that  the motivation for 
%overparameterization cannot be found  from a variational perspective, but perhaps it needs to be looked for in statistical or optimization reasoning.} 

A second line of inquiry regards the characterization of the functions and the norms corresponding to neural networks. 
Once again, it is instructive to look at the Hilbert setting. A main example of reproducing kernel Hilbert spaces are Sobolev spaces with  smoothness  sufficiently high for the embedding theorem to hold. In this case, the norm in the reproducing kernel Hilbert space can be characterized in terms of a suitable pseudo-differential operator, with the associated reproducing kernel being the corresponding Green function \cite{wendland2004scattered}. Again, the  question is whether similar characterizations can be derived for reproducing kernel Banach spaces defined by neural networks.
A recent line of works shows that results  in this direction can be derived when considering the rectified linear activation function (ReLU) in the network units.
A first result in this direction is derived in \cite{savarese2019infinite} for univariate functions, and then developed in \cite{ongie2019function} for the general multivariate case. In particular, this latter paper shows that the corresponding Banach semi-norm can be characterized using the Radon
transform. These  results are further developed in \cite{parhi2021banach}, where semi-norms are defined  in terms of the Radon transform in order to prove a representer theorem for one hidden layer neural networks with (generalized) ReLU activation function. In particular, the definition of the semi-norm precedes and is in function of proving the representer theorem. Here we contribute to this line of work, refining and extending such results, as well as providing different derivations. Indeed, we show that an analogous yet finer Radon characterization holds true for the reproducing kernel Banach spaces corresponding to neural networks with   (generalized) ReLU activation functions. Our construction shows that the characterization of the Banach space structure is independent of the representer theorem. Moreover, our approach provides a natural norm regularizer, thus avoiding semi-norms with resulting topological issues. Using a norm instead of a semi-norm also prevents the addition of null space elements ({\it i.e.} polynomials) to the neural network minimizers. 
We end noting that, while working on the characterization of the regularizer norm of ReLU neural networks, we also contribute to Radon transform theory, extending the classical inversion formulae to larger spaces of Lizorkin distributions.

The paper is organized as follows. In \Cref{sec:back} we recall the main ideas and results about learning with kernels. In \Cref{sec:representer} we give a short introduction to reproducing kernel Banach spaces (RKBS) and their characterization in terms of feature maps.
Then, we introduce a class of integral RKBS
that can model one hidden layer neural networks,
and establish a representer theorem for such a class in \Cref{subsec:representer}.
In \Cref{sec:reluradon} we focus on the special case of one hidden layer neural networks with (generalized) ReLU activation function.
In particular, in \Cref{sec:main2} we characterize the corresponding  norm by means of the Radon
transform.
In~\Cref{sec:radon} we review the theory of the Radon transform,
and we prove extensions of the classical Radon inversion formulae to Lizorkin distributions. \Cref{sec:proof-main} contains the proofs of the main results of \Cref{sec:main2}. {In Sections~ \ref{sec:discussion-3}, \ref{sec:discussion-2} and \ref{sec:discussion-4} we discuss and compare our results with \cite{parhi2021banach} and with previous work on representer theorems and Radon distributional theory.} Finally, in \Cref{appendix} we collect some variational results that we use to prove our representer theorem.

\vspace{.5cm}

\paragraph{\bf Notation.}
If $ x, y \in \mathbb{R}^d $, $ x\cdot y $ denotes their scalar product and $|x|$ denotes the Euclidean norm.
The length of a multi-index
$m \in \mathbb{N}^d $ is denoted by $ |m| = m_1 + \ldots + m_d $.
Furthermore, if $x=(x_1,\ldots,x_d)\in \mathbb{R}^d$ and
$m=(m_1,\ldots,m_d)\in\N^d$, we use the notation $ x^{m} = x_1^{m_1}
\cdots x_d^{m_d}$ and
$\partial^{m}=\partial_x^{m}=\partial_{x_1}^{m_1}
\dots\partial_{x_d}^{m_d}$. We denote by $S^{d-1} $ the unit
sphere in $\R^d$. The dual pairing between a locally convex
topological space $ {\mathcal A}$ and its topological dual space 
${\mathcal A'}$ is denoted by ${_{\mathcal A'}\langle\: \cdot\:,
  \:\cdot \:\rangle_{\mathcal A} }.$ For simplicity, we also write
the pairings without specifying the dual pair ${\mathcal A}, {\mathcal A'}$ whenever it is clear from the context. The Fourier transform $\mathcal F $ is defined for $ \varphi \in L^1(\R^d) $ by
\begin{equation*}
\mathcal F \varphi({\omega})= \frac{1}{(2\pi)^{d/2}}\int_{\R^d} \varphi(x) e^{-i\,
x\cdot\omega } \D{x} , \qquad  \omega \in \R^d ,
\end{equation*}
and it extends to $L^2(\R^d)$ in the usual way. 

If $\bb$ is a Banach space, we denote by $\nor{\cdot}{\bb}$ the
corresponding norm.  If $\mm$ and $\nn$ are two  subspaces of $\bb$, we write
$ \bb = \mm + \nn $ to mean that
\[
  \bb =\{ m+n\colon m\in \mm ,\ n\in \nn\} , \qquad \mm \cap \nn =\{0\},
  \]
%Then, there exist constants $C_1, C_2\geq 0$ such that
%\[ C_1\nor{m+n}{\bb}\leq \nor{m}{\bb}+\nor{n}{\bb}\leq C_2\nor{m+n}{\bb},\qquad \forall m\in \mm,\, n\in \nn.\]
and we denote by $P_{\mm}$ and $P_{\nn}$ the corresponding projections
\[
P_\mm, P_\nn: \bb \to \bb , \qquad P_{\mm}(m+n)= m, \quad P_{\nn}(m+n)= n, 
\]
so that $I=P_\mm+P_\nn$. If $\mm$ and $\nn$ are two Banach spaces, we write
$ \bb = \mm \oplus \nn $ to mean that product space $\mm\times \nn$ endowed with the $\ell^1$-norm
\[
\nor{m+n}{\bb} = \nor{m}{\mm}+\nor{n}{\nn} .
\]

%%%%%%%%%%%%%%%%%%

\section{Background: learning with ERM and RKHS}\label{sec:back}
%%%%%%%%%%%%%%%%%%%%%%%%%%%%%%%%%%%%%
In this section, we provide some background useful for the developments in later sections.
In particular, we recall the main ideas behind learning via empirical risk minimization (ERM) and the need of incorporating a bias in the search of a solution space. Further, we recall the basic ideas and results related to  considering reproducing kernel Hilbert spaces (RKHS) as solution spaces, in particular the representer theorem and the interpretation of the bias induced by the RKHS norm.

\subsection{Background: learning with ERM}
%%%%%%%%%%%%%%%%%%%%%%%%%%%%%%%%%
The basic problem of supervised learning is to estimate a function $f:\xx\to \R$ of interest, given a (training) set of input/output pairs $D=(x_1,y_1), \dots, (x_n,y_N) \in (\xx\times \R)^N$. The problem is formalized in the setting of statistical learning theory \cite{vapnik1998statistical,Cucker02onthe,gyorfi},  by assuming that $\xx\times \R$ is a probability space with distribution $P$ and that the training set is sampled identically and independently,  that is  $D\sim P^N$. Then, the function of interest is the one minimizing the expected risk
$$
\min_{f\in \mathcal T} {\mathcal L}(f), \quad {\mathcal L} (f)= \int L(y,f(x))dP(x,y),
$$
where $L:\R\times \R\to [0,\infty)$ is a given  loss function. 
Here,  the minimization is thought over the largest space  ${\mathcal T}$ over which the expected risk is defined.
% shall be more precise? 
We note that the expected risk can be interpreted as an idealization of the notion of test error. 
In practice, the minimization of the expected risk is unfeasible for at least two reasons. 
The first one is that the measure $P$ is known only through the training set $D$. The second one is that, in practice, the search of a solution needs to be  restricted to some class of functions $\hh\subset {\mathcal T}$, called hypothesis space. The natural approach is then to consider the empirical risk minimization
$$
\min_{f\in \hh}\widehat {\mathcal L}(f), \quad \widehat {\mathcal L}(f)= \frac 1 n \sum_{i=1}^NL(y_i,f(x_i)). 
$$
While the choice of $\hh$ might seem as a strong restriction, there are example of spaces such that 
$$
\min_{f\in \hh} {\mathcal L}(f)= \min_{f\in {\mathcal T}} {\mathcal L}(f),
$$
sometimes called universal classes of function \cite{steinwart2008support,carmeli2010vector}. As pointed out later, functions spaces used in both kernel methods and neural networks can be shown to have this property.
In this case, ERM is often modified considering 
$$
\min_{f\in \hh} 
\widehat  {\mathcal L}  (f) +  J(f),
$$
where $J:\hh\to \R$ is a functional, called regularizer. 
The idea is that the regularizer should enforce  a bias in the search of a solution in $\hh$ and  
help finding stable solutions. Next,  we discuss a classic example of hypothesis spaces and regularizers, useful in our discussion. 

\subsection{RKHS, representer theorem and regularizers}
We next consider the hypothesis space to be a RKHS. We begin recalling
a general  definition of RKHS and useful equivalent characterizations. 
\begin{dfn}\label{defn:rkhs}
 Let $\xx$ be a set. A \emph{reproducing kernel Hilbert space} (RKHS) $\mathcal{H}$ over  $ \xx $
 is a Hilbert space of functions $f:\xx\to\R$ such that:
 \begin{enumerate}[label=\textnormal{(\roman*)}]
 \item  as a vector space, $\hh$ is endowed with the pointwise
   operations of sum and multiplication by a scalar;
 \item for all $x\in\xx$, there is a constant $ C_x > 0 $ such that
   \begin{equation}
|f(x)|\leq C_x\|f\|_{\hh}, \qquad \forall f \in \hh .\label{eq:repro}
\end{equation}
 \end{enumerate}
\end{dfn}
The property \eqref{eq:repro}
states that, for every $ x \in \xx $,
the point evaluation functional
$ \operatorname{ev}_x : \hh \to \R $,
$ \operatorname{ev}_x f = f(x) $,
is continuous.
By the Riesz representation theorem,
\eqref{eq:repro} is thus equivalent to
the existence, for all $ x \in \xx $,
of an element $ K_x \in \hh $ such that
$ f(x) = \langle f , K_x \rangle_\hh $
for all $ f \in \hh $.
This observation leads to the following
more practical characterization of RKHS
\cite{MR51437}.
% While the above definition is general, other characterizations are often more practical.
% The following  result is classical \cite{MR51437}.
\begin{prop} \label{prop:RK}
 A Hilbert space $\hh$ of functions on $ \xx$ is a RKHS
 if and only if
 there exists a function  $K:\xx \times \xx \to \R$ such that  for all $x\in \xx$
  \begin{enumerate}[label=\textnormal{(\roman*)}]
  \item  $K(x,\cdot)\in \hh$, 
  \item  $ f(x)=\langle f, K(x,\cdot)\rangle_\hh, ~\forall f \in \hh$.
\end{enumerate}
\end{prop}
 The function $K$ is called the \emph{reproducing kernel} and item (ii) is called the \emph{reproducing property}.
It is easy to check that every reproducing kernel is symmetric and positive 
definite.  The reproducing kernel, often just called the kernel, is a key quantity uniquely 
associated to each RKHS. 
In the following, we will see how kernels are useful to characterize the solutions of corresponding ERM problems.  More generally,  it is possible to prove  a converse of the above result showing that each symmetric positive definite kernel can be used to  define a unique RKHS \cite{MR51437}. Here, we omit this characterization and recall another one which is popular in machine learning.
\begin{prop} \label{prop:feat}
 A space $\hh$ of functions on $ \xx$ is a RKHS
 if and only if
 there exist a Hilbert space $\mathcal{F}$ and a map $ \phi : \xx \to \mathcal{F} $ such that
 \begin{enumerate}[label=\textnormal{(\roman*)}]
  \item \label{it:rkhs-rep_prop}
  $ \hh = \{ f_w : w \in \mathcal{F} \} $ where $ f_w (x) = \langle \phi(x), w \rangle_{\mathcal{F}} $; \\
  \item \label{it:rkhs-norm}
  $ \| f \|_\hh = \inf \{ \| w \|_{\mathcal{F} }: w \in \mathcal{F} , f = f_w \} $.
\end{enumerate}
\end{prop}
The map  $ \phi $ is called a \emph{feature map} and  $\mathcal{F}$  a \emph{feature space}.
By the above result, each function in a RKHS can be seen as a hyperplane in the feature space. 
The linear parameterization of a RKHS is explicit in the above characterization.

An extension of Definition~\ref{defn:rkhs} and its equivalent characterizations will be useful in the following, while considering neural nets. 
It will also be useful to recall two immediate consequences. The first is a representer theorem that characterizes the solution of  the ERM regularized with the squared RKHS norm. 
\begin{thm} \label{thm:rkhs-representer}
Assume $\hh$ is a RKHS and, for every $ y \in \R$, the
function $ L(y,\cdot) $ is convex.
Then, the problem
$$
\min_{f\in \hh} 
\widehat {\mathcal L}  (f) +  \nor {f}{\hh}^2
$$
has  a unique minimizer $f^*$ such that, for all $x\in \xx$,
$$
f^*(x)=\sum_{i=1}^n K(x,x_i)c_i, \quad c_i\in \R.
$$
\end{thm}
% shall we add the geometric proof.?
The above result is remarkable since it implies that the minimization over an infinite dimensional space can 
be replaced with a finite dimensional one. Indeed, this is the key observation behind kernel methods \cite{scholkopf2002learning}. 
We end this section recalling that  for several reproducing kernels the nature of the regularizers induced by the corresponding squared RKHS  norm can be interpreted via  an equivalent characterization. 
\begin{ex}[Differential operators and Sobolev spaces]
Let  $\xx= \R^d$ and  $k(x,x')=e^{-\|x-x'\|}$ the Laplacian kernel. 
Then, for $s=d/2+1/2$, it can be shown that 
$$
\nor{f}{\hh}^2\asymp\nor{f}{2}^2+ \nor{\Delta^{s/2} f}{2}^2,
$$
where $\nor{f}{2}= \int |f(x)|^2dx$, and $\Delta$ is the Laplace Beltrami operator.
Through  the above characterization, functions with a small RKHS norm can be seen to be more regular.
Similar reasoning can also be shown to apply to other translation invariant kernels. 
Interestingly, for all these examples the corresponding RKHS are universal \cite{JMLR:v7:micchelli06a}.
\end{ex}

In the following we discuss  the question of whether the  above results apply or can be extended to neural networks, and discuss several implications.

\section{RKBS  of Neural networks and representer theorem}
%\section{Representer theorems on RKBS}
\label{sec:representer}

In this section,  we discuss how different function spaces can be associated to neural networks. In particular, we discuss how RKBS can be used towards this end, and corresponding representer theorems derived.  We first recall the basic expression for neural networks with one hidden layer and illustrate the benefit of considering the limit in which the hidden layer can   have infinite width.
% neural network.
% introduce a class of RKBS
%parametrized by the Banach space of bounded measures on a parameter
%space and we prove a representer theorem for such a class of
%spaces. We first recall basic definitions and properties of RKBS.

%are linearly parameterized over measures}

\subsection{Infinite wide neural networks are linearly parameterized over measures}

As mentioned before, a main obstacle towards  studying  function spaces defined by neural networks is their nonlinear parameterization.  Starting from a linear function $w \cdot x$, Proposition ~\ref{prop:feat} shows how  linearly parameterized  nonlinear functions can be  obtained applying a non linear map to the input $w \cdot \phi(x)$. 
In neural networks instead, a continuous nonlinear function  $ \si : \R \to \R$ is applied also to the parameters by considering $\si ( w\cdot x  )$. Indeed, this latter expression is a simplified model of a neuron. 
A one hidden layer neural network is a function obtained as  linear combination of several neurons
\begin{equation} \label{eq:nn2}
f(x) = \sum_{k=1}^K \al_k \si ( w_k \cdot x - b_k ) ,
\end{equation}
where  $ w_k \in \R^d $ and $  b_k \in \R $ are  called the weights.
The above expression can be developed considering further compositions to obtain ``deeper''
 multilayer architectures. In this paper, we restrict our attention to one hidden layer networks. 
 In the following, we discuss how functions spaces of neural networks can be defined 
 very generally considering an extension of RKHS, namely RKBS. We first
 illustrate   some  basic ideas, in particular a suitable reparameterization of neural networks in terms of measures. 
%  in particular the benfit
%We next recall how the above challenge can be tackled through a suitable 
%reparameterization (see \cite{bach and who else?}).  

We  use the short hand notation $\rho(x, \theta)=
\si ( w \cdot x - b ) $, where $\theta = (w,b)$.
The key idea is  to consider the limit for large $K$ in equation~\eqref{eq:nn2},  that is 
\begin{equation}\label{eq:infnn2}
 \sum_{k=1}^K  \rho(x, \theta_k) c_k\quad \mapsto \quad   \int \rho(x, \theta) d \mu(\theta).  
 \end{equation}
 The latter expression is the limit where the hidden layer has an infinite number of neurons. Note that, if $\mu=\sum_{k=1}^K \delta_{\theta_k} c_k$ then  
 $$\int \rho(x, \theta) d \mu(\theta)= \sum_{k=1}^K  \rho(x, \theta_k) c_k.
$$
Considering the integral in equation~\eqref{eq:infnn2} requires some care,  and in the next few sections we will discuss how function spaces can be defined with the aid of RKBS. We first discuss a simplified setting to illustrate some basic intuitions. 

\begin{ex}[Compact parameter spaces and densities]
We let  $\theta\in \Theta $,  where the parameter $\Theta$ is a compact subset of $\R^d$, and restrict to measures that are absolutely continuous with respect to the Lebesgue measure $d\theta$, so that $\mu(\theta)= p(\theta)d\theta$.
Then, equation~\eqref{eq:infnn2} becomes
\begin{equation}\label{eq:infnn}
f_\mu(x)= \int \rho(x,\theta)p(\theta)d\theta.
 \end{equation}
 The above expression shows how functions are linearly parameterized by measures/ densities, and it is easy to see that they form a linear space. Different structures and in particular 
different norms can be considered, for example
$\nor{f}{\hh} = \nor{p}{{L^2(\Theta)}}$
or 
$\nor{f}{\bb} = \nor{p}{{L^1(\Theta)}}$.
It can be proved   \cite{JMLR:v18:14-546,rudi2017generalization} that the first choice 
% $\nor{f}{\hh} = \nor{p}{L^2(\theta)}$ 
corresponds to considering a RKHS $\hh$ with kernel 
$$
K(x,x')= 
 \int \rho(x,\theta)\rho(x',\theta) d\theta. 
$$
Indeed, this result is at the base of well known  connections between neural networks with RKHS, and in particular random features \cite{NIPS2007_013a006f},  but also with Gaussian processes \cite{neal2012bayesian}.
The norm $\nor{f}{\bb} = \nor{p}{{L^1(\Theta)}}$, instead, can be shown to  define a Banach space \cite{JMLR:v18:14-546}, and clearly
$
\hh\subset \bb. 
$
Hence, in general, we can expect the space $\bb$ to have better approximation properties than $\mathcal{H}$.
We remark that,
while surely enlightening,
this setting has at least two major limitations:
first, the parameter space of commonly used neural networks is never compact;
second, restricting to absolutely continuous measures excludes atomic measures,
and therefore (finite width) neural networks.
\end{ex}
The above example shows that,  while a connection between RKHS and neural nets is possible, going  beyond a Hilbert setting might be needed depending on the kind of structures we consider on the function space of neural networks. The fact that Banach spaces of neural networks are larger function spaces suggests that it could be interesting to explore this setting.
Interestingly, recent results also suggest that the gradient descent training of neural networks  might be controlling implicitly the norm in $\bb$ \cite{chizat2020implicit}. Indeed, we will show next that certain Banach spaces are naturally associated to neural networks. Towards this end, we first recall the basic facts about RKBS.

%%%%%%%%%%%%%%%%%%%%%%

\subsection{Reproducing kernel Banach spaces}
 Since \cite{JMLR:v10:zhang09b}, several definitions of RKBS have been 
proposed.  Here, we adopt a fairly minimal definition,
and refer to \cite{Lin2019OnRK} for a comprehensive overview.
Among all possible equivalent definitions of RKHS,
the one in \Cref{defn:rkhs} generalizes seamlessly to the Banach case.
Indeed, it suffices to replace ``Hilbert'' with ``Banach''.
\begin{dfn}
 Let $\xx$ be a set. A \emph{ reproducing kernel Banach space} (RKBS) $\mathcal{B}$ over  $ \xx $
 is a Banach space $\bb$ of functions $f:\xx\to\R$ such that:
 \begin{enumerate}[label=\textnormal{(\roman*)}]
 \item  as a vector space, $\bb$ is endowed with the pointwise
   operations of sum and multiplication by a scalar;
 \item for all $x\in\xx$, there is a constant $ C_x > 0 $ such that
   \begin{equation}
|f(x)|\leq C_x\|f\|_{\bb}, \qquad \forall f \in \bb .\label{eq:42}
\end{equation}
 \end{enumerate}
\end{dfn}
As for RKHS, the property~\eqref{eq:42}  is equivalent to
the fact that for every $ x \in \xx $ there exists an element $ \operatorname{ev}_x \in
\mathcal{B}' $ such that
\begin{equation}
   \label{eq:1}
   f(x)= \scal{\operatorname{ev}_x}{f}{\bb'}{\bb}, \qquad \forall f\in \bb.
 \end{equation}
However, unlike for RKHS,
this does not lead to a natural notion of reproducing kernel,
and thus to a characterization as in \Cref{prop:RK},
because in general $\bb$ is not isomorphic to its dual.
Interestingly, the characterization of \Cref{prop:feat} in terms of feature maps generalizes naturally \cite{MR3716734, Lin2019OnRK}.
% For RKHS several characterizations and constructions are possible, but perhaps the most popular in machine learning is the one in terms of feature maps. The basic idea is that a feature map $\phi:\mathcal X \to \mathcal F$ provides a nonlinear representation of each input point in some suitable Hilbert space $\mathcal F$ called feature space. To each RKHS it is possible to associate a feature map (in fact, infinitely many) such that, for every function in the RKHS, the following representation holds: $f(x)= \langle \phi(x) ,  w\rangle$ for some $w\in \mathcal F$. Then, functions in the RKHS can be seen as  hyperplanes in the feature space. See {\it e.g.} \cite{MR2265340,scholkopf2002learning,steinwart2008support}. Interestingly, such a construction extends to RKBS,
% as shown in \cite{MR3716734, Lin2019OnRK}.
We report the proof for the sake of completness. 
\begin{prop} \label{prop:RKBS}
 A space $\bb$ of functions on $ \xx$ is a RKBS
 if and only if
 there exist a Banach space $\mathcal{F}$ and a map $ \phi : \xx \to \mathcal{F}' $ such that
 \begin{enumerate}[label=\textnormal{(\roman*)}]
  \item \label{it:rkbs-rep_prop}
  $ \bb = \{ f_\mu : \mu \in \mathcal{F} \} $ where $ f_\mu (x) = {}_{\mathcal{F}'}\langle \phi(x), \mu \rangle_{\mathcal{F}} $; \\
  \item \label{it:rkbs-norm}
  $ \| f \|_\bb = \inf \{ \| \mu \|_{\mathcal{F} }: \mu \in \mathcal{F} , f = f_\mu \} $.
\end{enumerate}
\end{prop}
\begin{proof}
 Let $\bb$ be a RKBS of functions on $ \xx $.
 Define $ \mathcal{F}= \bb $ and the canonical feature map
 \[
\phi: \xx \to \bb', \qquad \phi(x)=\operatorname{ev}_x,
   \]
where $\operatorname{ev}_x$ is defined by~\eqref{eq:1}, so that
$f_\mu=\mu$ for all $\mu\in \bb$. Both claims in the statement  are
clear.

Conversely, suppose we have a Banach space $\mathcal{F}$ and a map $ \phi : \xx \to \mathcal{F}' $, and define a vector space $\mathcal{B}$ of functions on $ \xx $ as in \ref{it:rkbs-rep_prop}. Then, the norm in \ref{it:rkbs-norm} makes $\bb$ a Banach space.
Moreover, in view of \ref{it:rkbs-rep_prop},
for every $ f \in \bb $ there exists $ \mu \in \mathcal{F} $ such that $ f = f_\mu $,
and $ | f(x) | = | f_\mu(x) | \le \|\mu\|_\mathcal{F} \|\phi(x)\|_{\mathcal{F}'} $.
Thus, for every $ x \in \xx $,
$$
|f(x)| \le \inf_{\mu\in \mathcal{F} , f = f_\mu} \|\mu\|_\mathcal{F} \|\phi(x)\|_{\mathcal{F}'} = \| f \|_\bb \| \phi(x) \|_{\mathcal{F}'} ,
$$
which shows that point evaluation is continuous on $\bb$.
\end{proof}
Some comments are in order. As mentioned above, Proposition~\ref{prop:RKBS} gives a recipe to construct RKBS starting from a Banach space $\mathcal{F}$ and a map $ \phi : \xx \to \mathcal{F}' $. In analogy to RKHS, we call  $ \phi $ a \emph{feature map} and  $\mathcal{F}'$  a \emph{feature space}.
As in the Hilbert setting, we note that  feature maps are in general not unique. Finally, we add a technical remark. 
\begin{rmk}
The RKBS $\bb$ is isometrically isomorphic to the quotient space $\ff/\mathcal
N$, where $\mathcal N$ is the closed subspace 
\[
  \mathcal N = \{ \mu\in\ff :  f_\mu (x) =0 \quad \forall x\in\xx \},
\]
and the isometry is given by
\[
W_\phi: \ff/\mathcal N\to  \bb,\qquad W_\phi( [\mu])= f_\mu ,
  \]
where $[\mu]$ is the coset of $\mu$. Since  the dual of
$\ff/\mathcal N$ can be identified with the 
  closed subspace 
  \[
\mathcal N^\perp=\{ \omega \in \ff' : \scal{\omega}{\mu}{\ff'}{\ff}=0
\,\forall \mu \in \mathcal N\}\subseteq \ff',
    \]
    then  by duality $ \bb'$  is  isometrically isomorphic to
    $\mathcal N^\perp$. In particular, 
  \begin{equation}
 W_\phi' \operatorname{ev}_x = \phi(x),  \qquad x\in\xx,\label{eq:10}
\end{equation}
where $ W_\phi' : \bb' \to \mathcal{N}^\perp $ denotes the dual map.
\end{rmk}
% Also note that, by construction, functions in the RKBS satisfy the \emph{reproducing property}
% $$
%  f_\mu(x) = {}_{\mathcal{F}'}\langle \phi(x), \mu \rangle_{\mathcal{F}} .
%  $$

Next, we  describe a class of RKBS parametrized by the space of bounded
measures,  which is a variant of an example in \cite{JMLR:v18:14-546}. 
This  RKBS is the example relevant to discuss   spaces of functions defined by neural networks.
% and prove a representer theorem for the  corresponding ERM solutions.

\subsection{A class of integral RKBS}\label{sec:class-integr-repr}
We fix  a (Hausdorff) locally compact second countable topological
space~$\Th$, that can be seen as the parameter space. Then,  we denote by  $
\mm(\Th)$ the Banach space of bounded measures defined on the Borel 
$\sigma$-algebra of $\Th$, and endow  $\mm(\Th)$ with the total
variation norm $\nor{\cdot}{\TV}$. Since $\Th$ is second countable,
the elements of $\mathcal M(\Th)$ are finite Radon measures and Markov-Riesz
representation  theorem ensures that  $\mm(\Th)$ can be identify with  the
dual of $\operatorname{C}_0(\Th)$, the Banach space of continuous
functions going to zero at  infinity endowed with the $\sup$~norm
$\nor{\cdot}{\infty}$.
Then the TV norm is written as
\begin{equation}
  \label{eq:11}
  \nor{\mu}{\TV}=
\sup\{  \scal{\mu}{ \psi}{}{} :
\psi\in\operatorname{C}_0(\Theta) , \|\psi\|_\infty\leq 1\}.
\end{equation}
Keys to our construction are 
%Fix a set $ \xx $, regarded as an input space, and take 
a function $ \rho : \xx \times \Th \to \R $ and a measurable function $ \be : \Th \to \R $ satisfying
the following conditions:
\begin{enumerate}[label=\textnormal{(\roman*)}] 
\item for all $ x \in \xx $
\begin{equation}
    \label{eq:7}
    \sup_{\theta\in\Th} |\rho(x,\th)  \beta(\th)|=D_x< \infty,
  \end{equation}
  for some $D_x>0$;
 \item for all $ x \in \xx $, the function $\rho(x,\cdot)$ is measurable.
\end{enumerate}
% \begin{equation}
%   \label{eq:2}
%   \lim_{\th\to\infty} \rho(x,\th) \beta(\th) = 0.
% \end{equation}
Given the above definition we next define a RKBS a functions with a suitable integral representation and that can be seen to be parameterized in terms of measures on the parameter space. As discussed later this yields a direct connection with one hidden layer neural networks with possibly infinite width.
Towards this end, we define 
$ $ 
% We consider the model
% \begin{equation} \label{eq:model}
%  f(x) = \sum_{k=1}^K \al_k \rho(x,\th_k) \qquad \al_k \in \R , \quad \th_k \in \Th , \quad K \in \N .
% \end{equation}
% Here $K$ represents the number of parameters:
% when $K$ is fixed a priori, we have a parametric model;
% if $K$ is allowed to vary as a function of $N$,
% we call \eqref{eq:model} a nonparametric model.
% We now construct a RKBS having \eqref{eq:model} as solutions.
%
the feature map
$$
 \phi : \xx \to \mm(\Th)',
 \qquad \scal{\mu}{\phi(x)}{\mathcal M(\Theta)}{\mathcal M(\Theta)'}=
 \int_\Th \rho(x,\th) \be(\th) \D \mu(\th)  , 
 $$
 which is well defined because of \eqref{eq:7}. Then, by~\Cref{prop:RKBS} the
 feature map $\phi$  defines a RKBS $\bb$ explicitly given by
\begin{subequations}\label{eq:9}
  \begin{align}
    \bb  & = \{ f_\mu : \mu \in \mm(\Th) \} \label{eq:RKBS}, \\
    f_\mu (x) & = \int_\Th \rho(x,\th) \be(\th) \D \mu(\th)  \label{eq:RKBS1}, \\
    \nor{f}{\bb} & = \inf\,\{ \nor{\mu}{\TV}   : f_\mu=f\} .  \label{eq:RKBS2} 
  \end{align}
\end{subequations}
We add several  remarks. 
First,  we comment on  the nature of the functions $\rho$ and $\beta$.
\begin{rmk}[Reproducing kernel and activation functions]
The function $\rho$ is a {\em reproducing kernel} in
the sense of \cite[Definition~2.1]{Lin2019OnRK}
(see \cite[Section 3.4]{Lin2019OnRK}).
% We will sometimes adopt this terminology, 
% albeit this notion of reproducing kernel  is quite different from that for RKHS. 
Clearly, it is always possible to include $\beta$ in
the definition of the kernel $\rho$.
However, we prefer to regard $\{\rho(\cdot,\theta)\}_\theta$
as a family of basis functions
({\it e.g.} as identified by the choice of an activation function in neural networks),
and $\beta$ as a smoothing function needed to ensure that the integral in~\eqref{eq:RKBS1} converges for all $\mu$.
% However, we prefer to  regard
% $\{\rho(\cdot,\theta)\}_\theta$ as a family of elementary generators and
% $\beta$ as a smoothing function needed to ensure that 
% the integral in~\eqref{eq:RKBS1} converges for all $\mu$. As we discuss later, in the case of neural networks, 
% the functions $\rho$ will be defined by an activation function.
\end{rmk}
As we comment next, the introduction of the smoothing function is crucial. 
\begin{rmk}[Smoothing function $\beta$] \label{rmk:beta}
Condition~\eqref{eq:7}  (with the measurability
assumption) is necessary and sufficient to ensure that the integral
in~\eqref{eq:RKBS1} converges for all bounded measures,
and thus that all the elements of the hypothesis space have an integral representation.
In \cite{parhi2021banach}
$\beta$ is not introduced,
and in fact their Lemma~21 provides an integral representation only for rapidly
decreasing measures.
Then, the authors assume that such a representation extends to a bounded operator.
Note, however, that
the extension 
of an integral operator is not necessarily integral.
For example, the $L^2$ extension of the Fourier transform does not admit an integral
representation.
On a related note,
\cite{unser2017splines} considers hypothesis spaces
with integral representation
by imposing a growth condition on the integral kernel (see Theorem 3 therein).
In our setting, such a kernel would correspond to the product of $\rho$ and $\beta$.
Since we need to keep $\rho$ free of growth conditions
(in order to plug in relevant examples of activation functions),
we charge $\beta$ with a decay condition.
In particular, our strategy allows to seamlessly deal with neural networks defined by ReLU activation functions.
\end{rmk}
By choosing  the measure $\mu$ having finite support, {\it i.e.}
\[
  \mu = \sum_{k=1}^K a_k\, \delta_{\theta_k} , \qquad  a_k \in \R , \quad
\th_k \in \Th,
\]
where $\delta_\theta$ is the Dirac measure at point $\theta$. It
follows that the elements of the form  
  \begin{equation}
f_\mu = \sum_{k=1}^K \al_k \rho(\cdot,\th_k) , \qquad \al_k=a_k\beta(\theta_k) \in \R , \quad
\th_k \in \Th, \label{eq:4}
\end{equation}
belong to $\bb$. Note that the smoothing function $\beta$ is included
in the vector coefficient $(\al_1,\ldots,\alpha_K)$, so that it does
not affect to the dependence of the function $f_\mu$ to the input
variable $x\in\xx$.  Functions as
in~\eqref{eq:4} are the main ingredient of many learning algorithms,
as for example kernel methods and one hidden layer neural networks,
see~\Cref{ex:kernel} and~\Cref{ex:nn} below.
As observed earlier, equation \eqref{eq:RKBS1} provides a
pointwise integral representation of the elements of $\bb$. However, by~\eqref{eq:4},
for each
$\theta\in\Theta$
  \begin{equation}
f_\theta=f_{\delta_{\theta}}= \rho(\cdot,\th) \be(\th) \in \bb, \qquad
\nor{f_\theta}{\bb}\leq \nor{\delta_{\theta}}{\TV}=1, \label{eq:14}
\end{equation}
  then
  \begin{equation}
    \label{eq:8}
    f_\mu = \int_\Th f_{\theta} \ \D \mu(\th) ,
  \end{equation}
  where the integral is in the Bochner sense provided that 
  $\theta\mapsto f_\theta$ is measurable as a map from  $\Th$ to $\bb$.  Finally, observe that~\eqref{eq:10} can be written as
  \[
W_\phi'\operatorname{ev}_x = \rho(x,\cdot)\beta  \in \mathcal M(\Theta)'.
    \]
    % \red{Last observation is obscure. Also shall we anticipate and give an example of $\rho,\beta$}?
    
\subsection{Representer theorem}\label{subsec:representer}

{We now derive a general representer theorem for the class of RKBS given
by~\eqref{eq:9}. As discussed next,  this amounts to providing explicit 
characterization of the solutions to  empirical minimization problems in machine learning and beyond.  
Following the setting described in \Cref{sec:back},
we consider the problem
\begin{equation} \label{eq:problem}
 \inf_{f\in \bb} \left(\frac{1}{N} \sum_{i=1}^N L ( y_i , f(x_i) ) +
   \|f\|_{\bb}  \right) .
\end{equation}
We are interested in the case where 
 the hypothesis space is the RKBS given
by~\eqref{eq:9} and $ \|  \cdot \|_\bb$ is the corresponding norm. With this choice, even existence of a solution  is non trivial
since in general $\bb$ is non-reflexive, so that the closed balls are not even weakly
compact. In the following we establish conditions under which minimizers exist, and derive a general representer theorem. }

First, we need a  result showing that  \eqref{eq:problem} can be reformulated 
as a minimization over the space of measures $\mathcal M(\Th)$.  The key observation is that  $\mathcal M(\Th)$ can be endowed with the weak$^*$ topology,
with respect to which the closed balls are indeed compact.
\begin{prop}\label{prop:existence}
Take $\rho:\xx\times\Th\to\R$,
$\beta:\Th \to \R$ satisfying~\eqref{eq:7}, and set $\bb$ as the
corresponding RKBS defined in~\eqref{eq:9}.
Then
\begin{equation*} %\label{eq:5}
 \inf_{f\in \bb} \left(\frac{1}{N} \sum_{i=1}^N L ( y_i , f(x_i) ) +
   \|f\|_{\bb}  \right) =\inf_{\mu\in \mathcal M(\Th)}
 \left(\frac{1}{N} \sum_{i=1}^N L ( y_i , f_\mu(x_i) ) +    \nor{\mu}{\TV}  \right).
\end{equation*}
Furthermore, if $\mu^*$ is any minimizer of 
  \begin{equation}
    \label{eq:3}
    \inf_{\mu\in \mathcal M(\Th)} \left(\frac{1}{N} \sum_{i=1}^N L ( y_i , f_\mu(x_i) ) +
   \nor{\mu}{\TV}  \right),
\end{equation}
then $f^*=f_{\mu^*}$ is a minimizer of  problem~\eqref{eq:problem}.
  \end{prop}
  \begin{proof}
By definition of $\bb$, we have
\begin{align*}
 \inf_{f \in \bb} \left(\frac{1}{N} \sum_{i=1}^N L ( y_i , f(x_i) ) + \|f\|_{\bb} \right)
 & = \inf_{\mu \in \mm(\Th)} \left(\frac{1}{N} \sum_{i=1}^N L ( y_i , f_\mu(x_i) ) + \|f_\mu\|_{\bb}\right) \nonumber \\
 & = \inf_{\mu \in \mm(\Th)} \left(\frac{1}{N}  \sum_{i=1}^N L ( y_i , f_\mu(x_i) ) + \inf_{ \substack{ \nu \in \mm \\ f_\nu = f_\mu } } \| \nu \|_{\TV} \right)\nonumber \\
 & = \inf_{ \substack{ \mu , \nu \in \mm(\Th) \\ f_\nu = f_\mu } }\left(
 \frac{1}{N} \sum_{i=1}^N L ( y_i , f_\mu(x_i) ) + \| \nu \|_{\TV} \right)\nonumber \\
 & = \inf_{ \nu \in \mm(\Th) }\left(\frac{1}{N} 
 \sum_{i=1}^N L ( y_i , f_\nu(x_i) ) + \| \nu \|_{\TV} \right) . %\label{eq:min2}
\end{align*}
Now let assume that $\mu^*$ is a minimizer
of~\eqref{eq:3}.
Then, for all $\nu\in\mathcal M(\Theta)$,
\[
\left(\frac{1}{N} 
 \sum_{i=1}^N L ( y_i , f_{\mu^*}(x_i) ) + \| \mu^* \|_{\TV} \right)
\leq \left(\frac{1}{N} 
 \sum_{i=1}^N L ( y_i , f_\nu(x_i) ) + \| \nu \|_{\TV} \right).
\]
Fix $\mu\in \mathcal M(\Theta)$ and take the infimum over all $\nu$
such that $f_\nu=f_\mu$, then
\[
\left(\frac{1}{N} 
 \sum_{i=1}^N L ( y_i , f_{\mu^*}(x_i) ) + \| \mu^* \|_{\TV} \right)
\leq \left(\frac{1}{N} 
 \sum_{i=1}^N L ( y_i , f_\mu(x_i) ) + \| f_\mu \|_{\bb} \right).
\]
With the choice $\mu=\mu_*$,  we have $\| \mu^* \|_{\TV} \leq
\| f_{\mu^*} \|_{\bb} $ and, clearly, $\| f_{\mu^*} \|_{\bb} \leq \| \mu^* \|_{\TV} 
$, so that
\[
\left(\frac{1}{N} 
 \sum_{i=1}^N L ( y_i , f_{\mu^*}(x_i) ) + \| f_{\mu^*} \|_{\bb} \right)
\leq \left(\frac{1}{N} 
 \sum_{i=1}^N L ( y_i , f_\mu(x_i) ) + \| f_\mu \|_{\bb} \right),
  \]
  which concludes the proof.
% Conversely, let $f^*$ is a minimizer,  by~\eqref{eq:6}
% \begin{align*}
%   \left(\frac{1}{N} 
%  \sum_{i=1}^N L ( y_i , f_{\mu^*}(x_i) ) + \| \mu^* \|_{\bb}
%   \right) & = \left(\frac{1}{N} 
%  \sum_{i=1}^N L ( y_i , f_{\mu^*}(x_i) ) + \| f_{\mu^*} \|_{\bb}
%             \right) \\
%   & \leq \left(\frac{1}{N}  \sum_{i=1}^N L ( y_i , f_\mu(x_i) ) + \|
%   f_\mu \|_{\bb} \right) \\
%   & \leq \left(\frac{1}{N}  \sum_{i=1}^N L ( y_i , f_\mu(x_i) ) +
%   \| \mu \|_{\bb} \right)
% \end{align*}
% for all $\mu\in\mathcal M(\Th)$. 
\end{proof}

The next corollary shows that the minimization
problem~\eqref{eq:3} can be regarded as two nested minimization
problems where  the external one is over a finite-dimensional vector
space. {As discussed in the following,  this result can be directly compared 
to the classic results for RKHS, highlighting similarities but also crucial differences.}
\begin{cor}\label{cor:repr-theor-rkbs}
 With the setting of~\Cref{prop:existence}, let
  \begin{equation}
 \mathcal V = \{ \mu\in\mathcal M(\Th) : f_\mu(x_i) = 0 \ \forall
 i=1,\ldots,N\} =\{ \rho(x_1,\cdot)\beta,\ldots,
   \rho(x_N,\cdot)\beta \}^\perp,\label{eq:12}
 \end{equation}
 where the orthogonal ${}^\perp$ is taken with respect to the pairing
$\scal{\cdot}{\cdot}{\mathcal M(\Th)'}{\mathcal M(\Th)}$.
%which is a closed subspace of $\mathcal M(\Th)$.
Then $\mathcal V$ is a closed subspace of $\mm(\Th)$,
and there exists a finite-dimensional subspace $ \mathcal W \subset \mm(\Th) $ with
$\dim{\mathcal W}\leq N$ such that
\[
\mathcal M(\Th) = \mathcal W + \mathcal V , % \qquad \mathcal W \cap \mathcal V =\{0\} ,
\]
 and 
  \begin{equation}
 \inf_{\mu\in \mathcal M(\Th)} \left(\frac{1}{N} \sum_{i=1}^N L ( y_i , f_\mu(x_i) ) +
   \nor{\mu}{\TV}  \right) = \inf_{\nu\in\mathcal W}\left(\frac{1}{N}
   \sum_{i=1}^N L ( y_i , f_\nu(x_i) ) +\inf_{\tau\in\mathcal V}
  \nor{\nu+\tau }{\TV}  \right).\label{eq:13}
\end{equation}
\end{cor}
\begin{proof}
  Define the map
$F:\mathcal M(\Th)\to\R$ ,
\[
F(\mu)= \left(\frac{1}{N} \sum_{i=1}^N L ( y_i , f_\mu(x_i) ) +
   \nor{\mu}{\TV}  \right).
\]
By the reproducing property~\eqref{eq:RKBS1}, the linear maps
  \[
    \mu \mapsto f_\mu(x_i) , \qquad i = 1 , \dots, N ,
  \]
  are continuous.
  Hence, $\mathcal V$ is a closed subspace of $\mathcal
  M(\Th)$ with finite co-dimension no larger than $N$,
  and therefore there is a finite
  dimensional subspace $\mathcal W$, $ \dim \mathcal W \le N $, such that
  \[
\mathcal M(\Th) = \mathcal W +  \mathcal V . %, \qquad \mathcal W\cap \mathcal V =\{0\}.
\]
Moreover, for all $\mu=\nu+\tau$ with $\nu\in \mathcal W$ and $\tau\in
\mathcal V$ , we have
\[
F(\mu)= \frac{1}{N} \sum_{i=1}^N L ( y_i , f_\nu(x_i) ) +
   \nor{\nu+\tau }{\TV},
\]
whence \eqref{eq:13} becomes clear.
\end{proof}

\Cref{cor:repr-theor-rkbs} is closely related to \Cref{thm:rkhs-representer}.
% showing that minimizers always belong to the subspace spanned by the kernel function evaluated at the input data points.
However, there are some important
differences. The existence of the finite-dimensional subspace $\mathcal
W$ strongly depends on the fact that $\mathcal V$ has finite
co-dimension. 
% \red{non ho capito }
Moreover, in general there is not a canonical choice
for the complement $\mathcal W$ 
and the total variation norm does not  preserve the decomposition,
{\it i.e.} in general  $\mathcal M(\Th)$ is  isomorphic to $ \mathcal
W \oplus\mathcal V$, but the isomorphism is not an isometry. For a RKHS
$\hh$, there is a canonical choice $\mathcal  W=\mathcal V^\perp$ and,
for such a choice,   $\nor{\nu+\tau}{\hh}^2 = \nor{\nu}{\hh}^2 + \nor{\tau}{\hh}^2$, so
that the inner minimization problem in~\eqref{eq:13}  has $\tau=0$ as
solution.
Further, since $\mathcal M(\Theta)$ is not reflexive, in
general $\mathcal V$ is only weakly closed (being convex), 
and it is not easy to show the existence of a minimizer for the inner minimization problem.

To overcome this issue, we  next strengthen condition~\eqref{eq:7} by
assuming that 
  \begin{equation}
    \label{eq:2}
    \rho(x,\cdot) \beta  \in \operatorname{C}_0(\Th) , \qquad \forall x\in\xx ,
  \end{equation}
  which clearly implies~\eqref{eq:7}.  This assumption is equivalent
  to assuming that the feature map 
  \[
\phi: \xx \to C_0(\Th)\subset\mm(\Th)'
\]
takes values in the pre-dual of $\mm(\Th)$ (compare with the assumption
in~\cite[Theorem~1, item 2]{unser2017splines}). Moreover, for
all $x\in\xx$,
\[
  W_\phi'\operatorname{ev}_x=  \rho(x,\cdot) \be \in \operatorname{C}_0(\Th) .
\]
We stress that, in many examples, given a function
$\rho$, it is easy to find a smoothing function $\beta$ such
that~\eqref{eq:2} holds true without modifying the form of the solutions \eqref{eq:4}
{(as functions of $x$)}.
{
On the other hand,
the choice of $\beta$ does affect the norm of the solutions, albeit in a simple way. Indeed, as seen later,  it simply corresponds to 
%although, as we will soon see, in a explicit and controlled way, namely 
renormalizing the coefficients.
}
Under condition~\eqref{eq:2},
we provide a representer theorem for the RKBS defined by~\eqref{eq:9}.
More precisely, we show that ERM minimizers always exist,
and are of the form \eqref{eq:4}.
Our proof takes care of some delicate topological issues
(see \Cref{rmk:sparse-solution}).
It is based on \cite[Theorem 3.3]{MR4040623},
the statement of which is given in~\Cref{sec:bredies-carioni} for the sake of completeness.

\begin{thm} \label{thm:representer}
Assume that~\eqref{eq:2} holds true and, for every $ y \in \R$, the
function $ L(y,\cdot) $ is convex and coercive in  the second entry. 
Then, the problem
\begin{equation*}
 \inf_{f\in \bb} \left(\frac{1}{N} \sum_{i=1}^N L ( y_i , f(x_i) ) +
   \|f\|_{\bb} \right)
\end{equation*}
admits solutions $f^*$ such that, for all $ x\in \xx$,
\begin{align}
  f^*(x) & = \sum_{k=1}^K \al_k \rho(x,\th_k) , \qquad \al_k \in
         \R\setminus\{0\} , \quad \th_k \in \Th , \label{eq:28} \\
   \| f^* \|_{\bb} & \le \sum_{k=1}^K | \al_k  \beta(\th_k)^{-1}| , \label{eq:29}
\end{align}
with $ K \le N$ and $\beta(\th_k) \ne 0$ for all $k=1,\ldots,
K$. 
\end{thm}
\begin{proof}
In view of \Cref{prop:existence} and~\eqref{eq:4},
to establish \eqref{eq:28}
it is enough to consider
the minimization problem~\eqref{eq:3} on the space $\mathcal
M(\Th)$,
and show that there exists a measure
$\mu$ with finite support of cardinality at most
$N$ that
minimizes \eqref{eq:3}. Towards this end, we apply~\Cref{thm:bredies-carioni}.

We set $U=\mathcal M(\Th)$ endowed with the weak$^*$ topology, so that
$U$ is a locally convex topological vector space. We define
\[ \aa: U \to \R^N , \qquad  (\aa \mu)_i = f_\mu(x_i)=\scal{\phi(x_i)}{\mu}{\mm(\Th)'}{\mm(\Th)}
= \scal{\mu}{\phi(x_i)}{C_0(\Th)'}{\C_0(\Th)} .\]
By \eqref{eq:2}, $\aa$ is a continuous linear
operator from $U$ to $\R^N$,  regarded as Hilbert space with respect
to the Euclidean scalar product. Furthermore,  by assumption on $L$,
the function
\[
F: \R^N \to (-\infty,+\infty], \qquad F(w)=\frac{1}{N} \sum_{i=1}^N L(y_i,w_i) , \qquad w=(w_1,\ldots,w_N)\in\R^N ,
  \]
is convex and coercive on $\R^N$ with domain $\R^N$,
thus it is continuous and, hence, lower semi-continuous.
We set $H=\operatorname{range}{\mathcal{A}}$, which is  a Hilbert space since it
a closed subspace of $\R^N$. With a slight abuse of notation, we regard
$F$ as a map defined on $H$  and $\aa$ as a map onto $H$, so that
$\aa$ becomes surjective. 
By~\eqref{eq:11}, the total variation norm, regarded as a seminorm from $U$ into
$(-\infty,+\infty]$,  is the superior envelope of lower semi-continuous
functions, hence it is weakly continuous \cite[page 11, item 4]{brezis},
its domain is $U$ and its kernel is trivial.   Furthermore, the Banach-Alaoglu theorem gives
that  the balls $ \{ \nu \in \mm(\Th) : \|\nu\|_{\TV} \le R \} $ are
weakly$^*$ compact for every $ R > 0 $, so that, according to the
definition in \cite[Assumption H1]{MR4040623}, the norm $ \| \cdot \|_{\TV} $ is
coercive on $U$.

By \Cref{thm:bredies-carioni},
the problem \eqref{eq:3} has minimizers of the form
\begin{equation*} %\label{eq:formnu}
 \mu = \sum_{k=1}^K a_k u_k,
 \qquad
 K \le N ,
 \quad
 a_k >  0 ,
 \quad
 \sum_k a_k = \| \mu \|_{\TV} ,
 \quad
 u_k \in \Ext (B) ,
\end{equation*}
where $B$ is the unit ball in $\mm(\Th)$ and $\Ext (B)$ is the set of
extremal points of $B$ (see~\Cref{dfn:extr-points-theor}).
Furthermore, thanks to~\Cref{lem:Ext(B)},
$$ 
\Ext (B) = \{ \pm \de_{\th} : \th \in \Th \} ,
$$
so that $\mu$ is a measure with finite support of cardinality at most
$N$.
We thus set $f=f_{\mu}$ and
\[
\alpha_k=
\begin{cases}
  a_k\beta(\th_k) & u_k=\delta_{\theta_k} \\
  -a_k\beta(\th_k) & u_k=-\delta_{\theta_k}
\end{cases}.
  \]
By~\eqref{eq:4} we have
$\alpha_k=a_k\beta(\th_k)\neq 0$ if and only if $\beta(\th_k)\neq 0$,
so that~\eqref{eq:29} holds true by removing the parameters $\th_k$
such that $\beta(\th_k)=0$,
as a consequence
of~\eqref{eq:RKBS2} and  the fact that $ \sum_k a_k = \| \nu \|_{\TV} $.
\end{proof}

\begin{rmk}
While our main motivation is supervised learning,
and thus we focus on minimizing objectives defined by loss functions,
it is clear from the working assumptions of \Cref{thm:bredies-carioni} that
\Cref{thm:representer} holds true
for more general variational problems,
arising from different choices of sampling $\aa$
and finite-data constraint $F$ (see \cite{MR4040623}).
\end{rmk}

\begin{rmk}
The above result is close to \cite[Theorem~1]{unser2017splines},
\cite[Theorem 1]{parhi2021banach}, where in both cases there is  an extra
polynomial term. It is also close to~\cite[Theorem 4.2]{MR4040623}, \cite[Section 4.1]{unser2020unifying},
that are stated for $\mm(\Th)$.
For further details and comparisons,
see Sections \ref{sec:discussion-2} and \ref{sec:discussion-3}.
\end{rmk}

\subsection{Neural Network RKBS}

We start discussing some examples illustrating how the above results specialize to neural networks (we further develop this discussion in later sections).

\begin{ex}[One hidden layer neural networks] \label{ex:nn}
Let $ \si : \R \to \R $ be a continuous (nonlinear) activation function.
A one hidden layer neural network is a function
\begin{equation} \label{eq:nn}
f(x) = \sum_{k=1}^K \al_k \si ( w_k \cdot x - b_k ) ,
\end{equation}
with $ w_k \in \R^d $ and $  b_k \in \R $.
Let $ \Th = \R^{d+1} $,
$ \rho(x,\th) = \si ( w \cdot x - b ) $ for $ \th = (w,b) $,
and pick a $ \beta $ satisfying \eqref{eq:2}.
%The feature map $ \phi(x)(\th) =  \rho(x,\th) \beta(\th) $ is universal \cite{cybenko89}.
Applying \Cref{thm:representer}, we obtain solutions of the form \eqref{eq:nn},
with $ K \le N $.
Typical examples of $\sigma$ are
sigmoidal functions,
{\it i.e.} functions satisfying $ \lim_{t\to-\infty} \si(t) = 0 $ and $ \lim_{t\to+\infty} \si(t) = 1 $,
and the widely used Rectified Linear Unit (ReLU)
$ \si(t) = \max\{0,t\} $.
It is well known that for all these choices of $\sigma$ the corresponding hypothesis classes are universal \cite{cybenko89,pinkus1999approximation}.
In Section \ref{sec:reluradon} we will be studying in full detail
the RKBS and corresponding norm associated with one hidden layer neural networks with (generalized) ReLU activation function.
\end{ex}

\begin{ex}[RBF networks \&
kernel mean embedding
] \label{ex:kernel}
Assume that $\xx$ is a compact topological space and $\kappa:\xx\times\xx
\to \R$ is a continuous semi-positive definite kernel.  For $\xx=\R^d$, a classic example is the Gaussian kernel $\kappa(x,x')=e^{-\|x-x'\|^2\gamma}$, which is also an example of Radial Basis Function (RBF) \cite{que2016back}.  Let $\hh$  be
the corresponding  reproducing kernel Hilbert space and $\bb$ be the
Banach space given by~\eqref{eq:RKBS} with the choice $\Th=\xx$,  $\rho=\kappa$ and $\beta=1$.
Equation \eqref{eq:14} gives that $f_x= \kappa(\cdot,x)=\kappa_x$ for all $x\in\xx$, so
that~\eqref{eq:8} becomes
\[
  f_\mu =\int_{\xx} \kappa_x \ \D \mu(x) \in \hh.
  \]
It is interesting to note that this is exactly the kernel mean embedding of $\mu$ (see  \cite{muandet2016kernel} and
references therein). Hence $\bb$ is a subspace of $\hh$ and,  since
 the kernel mean embedding is continuous from $\mathcal M(\Th)$ into
 $\hh$, the norm $\nor{\cdot}{\bb}$ is stronger than the norm induced
 by the scalar  product of $\hh$. For example, if the kernel $\kappa$ is
 characteristic \cite{muandet2016kernel}, the map $\mu\mapsto f_\mu$
 is injective, so that $\bb$ is isometrically isomorphic to $\mathcal
 M(\Th)$, which is not separable, whereas $\hh$ is separable since
 $\xx$ is. 
Still, \Cref{thm:representer} states the existence of solutions of
the form
\[
f = \sum_{i=1}^K \alpha_i K_{x'_i},
\qquad
x'_i \in \xx,
\]
with $ K\le N $.
Note however that \Cref{thm:representer}
does not imply that the points $ x'_i $
belong to the training set $ \{ x_i \}_{i=1}^N $.
\end{ex}
In later sections, we will further develop the study of RKBS corresponding to neural networks defined by generalized ReLU functions and characterize their norm.
Before that, we discuss the representer theorem we proved, reviewing classical as well recent related results.

\subsection{Discussion: representer theorems in learning, Banach and variational theory} \label{sec:discussion-2}

The  representer theorem originates from the work of \cite{kimeldorf1970correspondence,kimeldorf1971some}
on interpolation and smoothing problems in reproducing kernel Hilbert spaces,
and plays a key role in  kernel methods \cite{scholkopf2001generalized,scholkopf2002learning}.
In a simple form, the classical representer theorem asserts that
the solution of the regularized empirical risk minimization on a RKHS
is a finite linear combination of the kernel evaluated at the input data  points.
This result is both conceptually and practically remarkable,
since it allows to compute the solution of an infinite-dimensional models solving a finite dimensional problem.

In a broader sense, the representer theorem can also be seen as a sparsity result, showing the existence of solutions that are combinations of at most as many elements as the number of samples,
regardless of how high the dimension of the hypothesis class is.
Sparsity is an important property in machine learning (as well as in signal processing),
and  can be enforced by constraining the $\ell^1$ norm of the model parameters \cite{tibshirani1996regression,chen2001atomic}.
In a finite-dimensional model, sparsity is essentially a consequence of Carath\'eodory's convex hull theorem (see {\it e.g.} \cite[Section B.1]{rosset2004boosting}).
Sparse models naturally generalize to infinite dimensions by 
replacing the linear coefficients with the integration with respect to a measure,
and the $\ell^1$ norm with the $\TV$ norm.
Along these lines, \cite{JMLR:v18:14-546,rosset2007} consider
superpositions of infinitely many (and more than countable) features
with $ \TV $ regularization.
\cite[Theorem 1]{rosset2007} can be seen as a representer theorem for bounded features and positive measures, based on an extension of Carath\'eodory's theorem to positive measures \cite[Theorem 2]{rosset2007}.
Note that these constructions go beyond kernel methods and RKHS,
and in particular in the direction of neural networks as described in previous sections,
hence requiring different tools from functional analysis.
% On the other hand, they fall outside the setting of RKHS,
% requiring different tools from  functional analysis.

The approach relevant to our study is given by reproducing kernel Banach
spaces. The paper~\cite{JMLR:v10:zhang09b} introduces reflexive RKBS and
proves a representer theorem (Theorem 19) for minimal norm
interpolation on uniformly convex RKBS (assuming linearly independent
features at the sample points). A different approach is given in \cite{MR3716734}. 
Uniform convexity is assumed so that the Riesz representation theorem holds,
thus ensuring that continuous linear functionals are semi-inner products.
Using bilinear forms instead of inner products,
\cite{song2013reproducing,song2011reproducing} handle non-reflexive spaces,
and study in particular RKBS with $\ell^1$ or $\TV$ norm.
Their construction starts directly from a kernel function,
on which they impose admissibility conditions to obtain representer theorems,
see
\cite[Theorem 4.8, Corollary 4.9]{song2013reproducing}, \cite[Theorem 2.4]{song2011reproducing}.
Non-reflexive $p$-norm RKBS are constructed in \cite{xu2019generalized} via generalized Mercer kernels,
although the representer theorems require reflexivity.
Further definitions of RKBS are reviewed and unified in \cite{Lin2019OnRK}.
While the authors provide a general framework to construct RKBS and kernels by pairs of feature maps,
their representer \cite[Theorem 4.4]{Lin2019OnRK} still assumes reflexivity of the feature space.
We remark that even in the non-reflexive spaces considered in \cite{song2013reproducing,song2011reproducing}
the kernel is a function on the square of the input space,
and therefore the model can not accomodate typical basis functions parameterized by a different parameter space than the input space,
thus ruling out integral feature models \cite{JMLR:v18:14-546,rosset2007} and neural networks.

The full generality of representer theorems beyond reflexive spaces can be found in optimization and variational theory,
where they have come to mean virtually any result establishing the existence of sparse solutions to empirical minimization problems with convex regularization.
This kind of problems has  a long history.
A notable example is Radon measure recovery with $\TV$ regularization,
for which ante litteram representer theorems (for bounded domains) can be found in \cite{fisher1975spline,zuhovickii1948remarks},
stating the existence of solutions that are finite linear combinations of Dirac deltas.
The proof of these results are crucially based on the Krein--Milman theorem
and the characterization of extremal points.
A more general setting has been recently developed in \cite{unser2017splines}.
Here, the authors start from a pseudo-differential operator $\operatorname{L}$,
and consider the inverse problem over an associated native space $\mm_{\operatorname{L}}$ of functions on $\R^d$ with generalized TV seminorm $ \| \operatorname{L} \cdot \|_{\TV} $.
Then, they show that the extremal points of such a problem are $\operatorname{L}$-splines,
{\it i.e.} functions which are sparsified by $\operatorname{L}$,
plus a term in the (finite-dimensional) kernel of $\operatorname{L}$.
This point of view has been considered by \cite{parhi2021banach} and extended
from $\R^d$ to $\P^d$ with the notion of ridge spline,
of which ReLU neural networks are examples. The papers~\cite{boyer2019representer,MR4040623} introduce an extremely general variational framework that extends \cite{unser2017splines}
to inverse problems on locally convex spaces with abstract convex \cite{boyer2019representer} or seminorm \cite{MR4040623} regularization.
The corresponding representers are established:
\cite[Theorem 1]{boyer2019representer} assumes a priori the existence of minimizers and focuses on the geometry of the solution set,
whereas \cite[Theorem 3.3]{MR4040623} provides sufficient topological conditions for the existence of minimizers.

In summary, we can roughly identify three lines of work studying representer theorems:
representers for learning models (classically kernel methods, more recently neural networks),
representers for RKBS (generalizing RKHS),
and representers in variational theory.
Recently, the abstract variational framework has been applied and
reconnected to machine learning. The paper~\cite{unser2020unifying} proves a general representer theorem for dual pairs of Banach spaces,
which can be specialized to a wide range of learning problems, including sparse regularization on non-reflexive spaces (using \cite[Theorem 1]{boyer2019representer}).
In \cite{parhi2021banach}, \cite[Theorem 4.2]{MR4040623} is applied to provide a representer theorem for neural networks with ReLU (type) activation function.
In our paper, we further incorporate and exploit the ingredient of (non-reflexive) RKBS.
While the RKBS structure is implicitly present in several previous works \cite{rosset2004boosting,JMLR:v18:14-546,parhi2021banach},
its role in the explicit construction and characterization of neural network models was not clear or emphasized.
In our work, we show how such a structure allows to directly derive representer theorems for feature models and neural networks
from general variational theory.
For a detailed comparison between our results and \cite{parhi2021banach} we refer to~\Cref{sec:discussion-3}.

\section{Banach representation and Radon regularization of ReLU neural networks}\label{sec:reluradon}
In this section we discuss the RKBS associated with truncated power activation
functions, including the ReLU.
This is related to the results in \cite{parhi2021banach},
but here we follow a dual
approach and provide a finer characterization.
First, we define a hypothesis space $\bb_m$ as a RKBS parametrized
by $\mathcal M(\Theta)$ for a suitable choice of $\Th$ and $ \rho = \rho_m $. Then, we
characterize the norm of $\bb_m$ by means of the Radon transform.

\subsection{The hypothesis space}
Let $S^{d-1}$ be the unit
sphere in $\R^d$, and let
$$
 \Xi= S^{d-1}\times \R
$$
with the product
topology, which makes it a locally compact second countable space.
Given $\mu\in\mathcal M(\Xi)$,
we set $\mu^\vee\in\mm(\Xi)$ to be the bounded measure defined by
  \begin{equation*}
\mu^\vee(E)=\mu(-E) % \label{eq:40}
\end{equation*}
  for every Borel set $E \subset \Xi$.
We define the subspaces of even and odd measures as
  \begin{align*}
    & \mm(\Xi)_{\rm even} =\{ \mu\in\mm(\Xi) : \mu^\vee=\mu\} , \\
    & \mm(\Xi)_{\rm odd} =\{ \mu\in \mm(\Xi) : \mu^\vee=- \mu \}.
  \end{align*}
Furthermore,  for every $\mu\in \mm(\Xi)$, we define the even and odd part of $\mu$ as
\begin{equation*}
\mu_{\rm even}=\frac{\mu+\mu^{\vee}}{2}\in \mm(\Xi)_{\rm even},\qquad
\mu_{\rm odd}=\frac{\mu-\mu^{\vee}}{2}\in \mm(\Xi)_{\rm
  odd}. %\label{eq:34}
\end{equation*}  
% We introduce on $\mm(\Xi)$ the norm 
% \begin{equation}
%     \nor{\mu}{\TVS} =\nor{\mu_{\rm
%         even}}{\TV}+\nor{\mu_{\rm odd}}{\TV},\label{eq:37}
% \end{equation}
% which is well-defined and equivalent to the total variation norm, as
% shown by~\Cref{lem:even-odd}. 
Every $\mu\in \mm(\Xi)$ can be written as the sum $\mu=\mu_{\rm even}+\mu_{\rm odd}$ and this factorization is unique, so that
\[
\mm(\Xi)= \mm(\Xi)_{\text{even}}+ \mm(\Xi)_{\text{odd}}.
\]

Moreover,  for every integer $m
\ge 2$, we
define the truncated power activation  function $ \sigma_m\colon\R\to  \R$ as  
\begin{equation} \label{eq:sigma_m}
 \sigma_m (t) = \frac{1}{(m-1)!} \max\{ 0 , t \}^{m-1},\qquad t \in \R
\end{equation}
(see~\Cref{fig:ReLUs}), and the correspondingly
\begin{equation} \label{eq:rho_m}
\rho_m:\R^d \times \Xi \to \R , \qquad \rho_m(x,n,t)=\sigma_m(n\cdot x -t).
\end{equation}
Note that,  for $m=2$, $\sigma_2$ corresponds to the Rectified Linear
Unit (ReLU).

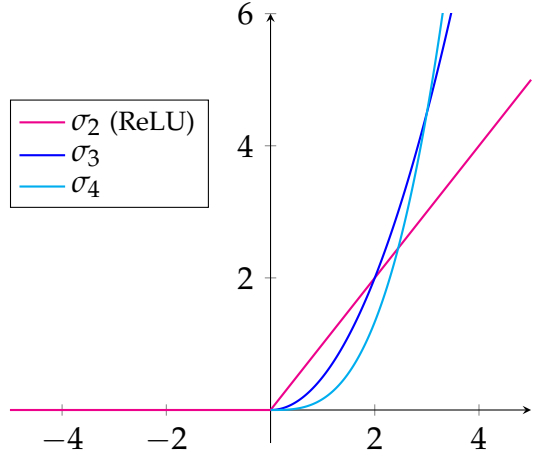
\begin{figure}[ht]
    \centering
    \begin{tikzpicture}
        \begin{axis}[ymin = -0.5, ymax=6, axis lines = middle,legend style={at={(0,0.8)},anchor=north west},legend cell align={left}]
            \addplot[domain=0:5, samples=200, thick, magenta] {x};
            \addplot[domain=0:4, samples=200, thick, blue] {1/2*x^2};
             \addplot[domain=0:4, samples=200, thick, cyan] {1/6*x^3};
            \addplot[domain=-5:0, samples=200, thick, magenta] {0};
            \addlegendentry{$ \si_2 $ \footnotesize (ReLU)}
              \addlegendentry{$ \si_3$}
              \addlegendentry{$ \si_4$}
        \end{axis}
    \end{tikzpicture}
    \caption{ReLU-type activation functions: the ReLU $\si_2$, and the truncated power functions $\si_3$ and $\si_4$.}
    \label{fig:ReLUs}
\end{figure}

We choose $\beta\in C_0(\Xi)$ such that
\begin{subequations}\label{eq:31}
  \begin{align}
    & 
      \beta(n,t)>0, \qquad \forall (n,t)\in \Xi, \label{eq:30} \\
     & \beta(-n,-t)= \beta(n,t), \qquad \forall (n,t)\in \Xi \label{eq:23} ,\\
   &  \lim_{t\to\pm\infty}  (|x|+|t|)^{m-1}\sup_{n\in S^{d-1}}
  \beta(n,t)  =0, \qquad \forall x \in \R^d \label{eq:20} .
  \end{align}
\end{subequations}
The positivity condition \eqref{eq:30} is posed
to characterize the kernel of the RKBS parametrization $ \mu \mapsto f_\mu $ (see \Cref{support}).
The symmetry requirement~\eqref{eq:23}
allows to {control} the parity
when dealing with Radon transform and measures  (see \Cref{lem:green} and \Cref{rmk:parity}).
The requirement \eqref{eq:20} ensures  that condition~\eqref{eq:2} holds true
(see \Cref{rmk:beta}),
since
  \begin{equation}
\sup_{n\in S^{d-1}}  \rho_m(x,n,t) \le \frac{1}{(m-1)!}
(|x|+|t|)^{m-1}.\label{eq:21}
\end{equation}
An example of $\beta$ satisfying the above conditions  is
\[
\beta(n,t) =\frac{1}{1+|t|^m}.
  \]

According to the framework of~\Cref{sec:class-integr-repr}, with the
choice of $\xx=\R^d$ as input space and 
$\Th=\Xi$ as parameter space,
we define  $\bb_m$ as the RKBS with
kernel $\rho_m$ and smoothing function $\beta$,  {\it i.e.}
\begin{subequations} \label{eq:15}
  \begin{align}
  & \bb_m = \{ f_\mu : \mu \in \mm(\Xi) \}  \label{eq:16} , \\
  & f_\mu (x) = \int_{\Xi} \sigma_m( n \cdot x - t)\beta(n,t) \
               \D\mu(n,t) \label{eq:17} , \\
    & \| f \|_{\bb_m} = \inf \{ \| \mu \|_{\TV} : \mu \in \mm(\Xi) , f
                     = f_\mu \}. \label{eq:normrkbsrelu}
  \end{align}
\end{subequations}

\subsection{The regularization norm} \label{sec:main2}
The next theorem  provides an alternative characterization of the norm \eqref{eq:normrkbsrelu}
by means of the  Radon transform.
A similar result is stated in
\cite{parhi2021banach}, within a different framework.
To state our result,
we first need to specify a few operators.
We list them here, and we refer to \Cref{sec:radon} for all the details.
The operator $\rr$ denotes the Radon transform from the space $\cS_0'(\R^d)$ of
Lizorkin distributions on $\R^d$ onto the space $\cS_0'(\Xi )$ of
Lizorkin distributions on the space $\Xi$ (Definitions
\ref{defn:classicalradon} and  \ref{defn:radonanddualdistributions}).
The operator $\Lambda^{d-1}$ is the Fourier multiplier defined
by~\eqref{eq:27} and \eqref{eq:La-ext}, and it is at the root of the inversion formulae for
the Radon transform
(\Cref{teo:bagkprojectionformulaclassical}
and \Cref{cor:inversionformuladistributions}).
The operator $\partial_t$ is the distributional derivative acting on the variable $t$
defined in~\Cref{prop:operatorA}.

  \begin{thm}\label{main2}
 Fix an integer $m\geq 2$. Set $\bb_m$ as the reproducing kernel 
 Banach space with $\rho_m$ as in \eqref{eq:sigma_m}, \eqref{eq:rho_m}
and $\beta$ satisfying~\eqref{eq:31}, and let $\mathcal Q_m$ and $\mathcal P_m$ be the subspaces defined by
 \begin{align*}
& \mathcal Q_m = \{ f_\tau\in\bb_m\colon\tau\in\mm(\Xi) , \ \tau^\vee=(-1)^m\tau \} ,\\
&  \mathcal P_m = \{ f_\nu\in\bb_m\colon\nu\in\mm(\Xi) , \
   \nu^\vee=(-1)^{m+1}\nu\} .
  \end{align*}
Then $\mathcal Q_m$ and $\mathcal P_m$ are closed   subspaces of
$\bb_m$  such that
\begin{align*}
  & \bb_m = \mathcal Q_m + \mathcal   P_m ,  %\label{eq:45} \\
  \end{align*}
and
  \begin{align*}
  & \mathcal   P_m =\{p:\R^d\to\R\colon p \text{ is a polynomial of degree at most }m-1\}.
  %=\operatorname{span}\{(x\cdot   n-t)^{m-1}\colon (n,t)\in \Xi\} . %\label{eq:36}
\end{align*}
Moreover:
 \begin{enumerate}[label=\textnormal{(\roman*)}]
 \item\label{item:5}
   the elements $ f \in \bb_m $ are continuous functions satisfying the growth condition
\begin{equation} \label{eq:growth_condition}
|f(x)| \leq C_f (1+|x|)^{m-1} , \qquad
     x \in \R^d ,
   \end{equation}
so that $f\in \cS_0'(\R^d)$;  
\item\label{item:6} for all  $\mu\in \mm(\Xi)$, % such that 
%   \[
% f_\mu = g \oplus p \qquad g \in \mathcal Q_m,\ p \in \mathcal P_m,
%     \]
setting
   \begin{equation}\label{eq:38}
     \tau= \frac{\mu+(-1)^m \mu^\vee }{2} , \qquad \nu
     =\frac{\mu+(-1)^{m+1} \mu^\vee }{2} , 
   \end{equation}
  we have
   \begin{equation*}
     P_{\mathcal Q_m} f_\mu =f_\tau ,  \qquad P_{\mathcal P_m} f_\mu =f_{ \nu}, % \label{eq:32}
   \end{equation*}
   and
%   \item for all  $\mu\in \mm(\Xi)$,
% such that  
% \[
% f= g \oplus p \qquad g \in \mathcal Q_m,\ p \in \mathcal P_m,
%     \]
  \begin{align}
  & \frac{1}{2 (2\pi)^{d-1}\beta}\partial_t^{m} \Lambda^{d-1}\mathcal{R} f_{\mu}=\tau\label{eq:24};
  \end{align}
\item for all $f\in \bb_m$ ,
\begin{align}
&   \nor{f}{\bb_m} \leq \nor{P_{\mathcal Q_m}f }{\bb_m} + \nor{P_{\mathcal P_m}f}{\bb_m} \leq 2 \nor{f}{\bb_m} \label{eq:33} , \\
 &  \nor{P_{\mathcal Q_m}f}{\bb_m}  =\nor{\frac{1}{2 (2\pi)^{d-1}\beta}\partial_t^{m} \Lambda^{d-1}\mathcal{R} f}{\TV} \label{eq:46} , \\
  & \nor{P_{\mathcal P_m} f}{\bb_m}  =\inf\{ \nor{\nu}{\TV} \colon \nu\in\mm(\Xi),\,
                   \nu^\vee=(-1)^{m+1}\nu,\, f_\nu=P_{\mathcal P_m} f \} \label{eq:47} ;
\end{align}
% so that
% \begin{align}
%      f & = g \oplus p \qquad g \in \mathcal Q_m,\ p \in \mathcal P_m 
% \end{align} 
   \item\label{item:4} take a tempered distribution $T\in \cS'(\R^d)$ such that
   \begin{align}
    \tau & =\frac{1}{2(2\pi)^{d-1}\beta }\partial_t^{m} \Lambda^{d-1}\mathcal{R} T\in \mm(\Xi)\label{eq:100} , \\
      T&-  f_\tau  \in \mathcal P_m \label{eq:101}
 %     \sum_{\ell=1}^L \alpha_\ell  (x \cdot n_\ell-t_\ell)^{m-1} \qquad \text{in } \cS'(\R^d)  ,
   \end{align}
%where $\alpha_1,\ldots,\alpha_L\in\R$ and
%$(n_1,t_1),\ldots,(n_L,t_L)\in \Xi$,  
then  $T\in\bb_m$  and
\[
P_{\mathcal Q_m} T = f_\tau,  \qquad P_{\mathcal P_m} =f_\nu, %=  \sum_{\ell=1}^L \alpha_\ell  (x\cdot   n_\ell-t_\ell)^{m-1} .
  \]
  for some $\nu\in\mathcal M(\Theta)$ such that $\nu^\vee=(-1)^{m+1}\nu$.
 \end{enumerate}
\end{thm}

The proof of \Cref{main2} is given in \Cref{sec:proof-main}.
Here we add some comments. Assume that  $m$ is even, in particular
$m=2$ for the ReLU (for odd $m$, simply interchange ``even'' and
``odd'' in what follows).  The measures $\tau$
and $\nu$ defined by~\eqref{eq:38} are the even and odd parts of $\mu$ and
\Cref{main2} states that
  \begin{align}
\bb_m = &\{ f_\tau \colon \tau\in\mm(\Xi)_{\rm even} \} + \{
  f_\nu\colon \nu\in\mm(\Xi)_{\rm odd}\} , \label{eq:35}
  \end{align}
so that any $f\in \bb_m$ admits a unique decomposition $f=f_\tau + f_\nu$
with $\tau\in \mm(\Xi)_{\rm even}$ and $\nu\in \mm(\Xi)_{\rm odd}$.
The even part $\tau$ is uniquely determined by the Radon transform of $f$
via \eqref{eq:24},
and $\nor{f_\tau}{\bb_m}=\nor{\tau}{\TV}$,
so that $\mathcal Q_m$ is isometrically isomorphic to $\mm(\Xi)_{\rm
  even}$.  The odd part $\nu$ over-parametrizes the finite-dimensional
space $\mathcal P_m$ of polynomials of degree less than $m$
%spanned by the generators
%$p_{n,t}(x)=(x\cdot n -t)^{m-1}$, with $(n,t)$ ranging in $\Xi$ 
and,  in particular, $\nor{f_\nu}{\bb_m} \leq  \nor{\nu}{\TV}$. 
% and
%   \begin{align*}
%     \nor{f}{\bb_m} & = \nor{f_\tau}{\bb_m}+\nor{f_\nu}{\bb_m} \\
%     & = \frac{1}{\red{2(2\pi^{d-1})}}\nor{\frac{1}{\beta }\partial_t^{m}
%       \Lambda^{d-1}\mathcal{R} f}{\TV},+\nor{f_\nu}{\mathcal P_m}.
%   \end{align*}
 Finally, 
let $L=\dim(\mathcal P_m)$, and let $p_1,\ldots ,p_L$ be an
algebraic basis of $\mathcal P_m$. Since  $L$ is finite-dimensional, there exists a dual family $q_1,\ldots ,q_L$ in $\bb_m'$ such that
\[ \scal{q_\ell}{p_{\ell'}}{\bb_m'}{\bb_m}=\delta_{\ell,\ell'} . \]
  Then, for all $f\in \bb_m$,
  \[
f_\nu = \sum_{\ell=1}^L  \scal{q_\ell}{f}{\bb_m'}{\bb_m}\, p_\ell.
    \]
\Cref{item:4} provides an equivalent characterization of $\bb_m$ as a
subspace of the space of distributions,
as it happens for Besov spaces \cite{triebel},
and it is closely related to the original approach in
\cite{parhi2021banach,unser2017splines}. Equation \eqref{eq:100} means that
there exists a bounded measure $\tau\in\mm(\Xi)_{\rm even} $ such that
\[
\frac{1}{2(2\pi)^{d-1}}\partial_t^{m} \Lambda^{d-1}\mathcal{R} T =\beta \tau
\qquad \text{in } \cS_0'(\Xi) .
\]
Thus, $f_\tau\in \mathcal Q_m\subset \bb_m\subset
\cS'(\R^d)$, and~\eqref{eq:101} is equivalent to assuming that the remainder 
$T-  f_\tau$  is a polynomial of degree less than $m$. Without assuming \eqref{eq:101}
we have the following result,
whose proof is postponed to \Cref{sec:proof-main}.
\begin{cor}\label{without}
Take a tempered distribution $T\in \cS'(\R^d)$ such that~\eqref{eq:100} holds
true.
Then there exist a unique $f\in \mathcal Q_m$ and a unique polynomial
$p$ such that $T=f+p$. 
\end{cor}
In \cite{parhi2021banach,unser2017splines}, the polynomial degree is enforced to be smaller than $m$ by requiring that $T$ is a distribution
satisfying the growth condition~\eqref{eq:growth_condition}.
Note that $ \bb_m $ satisfies \eqref{eq:growth_condition} by construction.

Finally,
we note that \Cref{thm:representer} immediately gives the following
representer theorem. 
\begin{cor} \label{cor:relurep}
Assume that, for every $ y \in \R$, the loss function $ L(y,\cdot) $ is
convex and coercive in  the second entry, and set $\bb_m$ as
in~\Cref{main2}. Then, the problem   
\begin{equation} \label{eq:problemReLU}
  \inf_{f\in \bb_m}\left(   \frac{1}{N} \sum_{i=1}^N L ( y_i , f(x_i) ) + \|f\|_{\bb_m} \right)
  \end{equation}
 always has minimizers of the form
  \begin{equation}\label{eq:solutionsReLU}
    f(x) = \sum_{k=1}^K \al_k \sigma_m ( n_k \cdot x - t_k ),
  \end{equation}
  where $ K \le N $, $(n_k,t_k)\in S^{d-1}\times \R$, $\al_k\in\R\setminus\{0\}$
  and
  \[\|f\|_{\bb_m}\leq \sum_{k=1}^K |\al_k|\beta(n_k,t_k)^{-1} . \]
\end{cor}

\begin{rmk}
As already observed in \cite[Lemma 25]{parhi2021banach}, the Radon regularization corresponds to several forms of coefficient regularization, such as 
$\ell^1$-path-norm \cite{NIPS2015_eaa32c96} and weight decay \cite{NIPS1991_8eefcfdf}. 
Indeed, if we take $f\in\mathcal{B}_m$ of the form
  \begin{equation}
    f(x) = \sum_{k=1}^K \al_k \sigma_m ( n_k \cdot x - t_k ),
  \end{equation}
  where $ K \in\N $, $(n_k,t_k)\in S^{d-1}\times \R$, $\al_k\in\R\setminus\{0\}$, a simple computation gives that 
\begin{align*}
\|P_{\mathcal{Q}_m} f\|_{\mathcal{B}_m}&=\nor{\frac{1}{2 (2\pi)^{d-1}\beta}\partial_t^{m} \Lambda^{d-1}\mathcal{R} f}{\TV}=\sum_{k=1}^K |\al_k|\beta(n_k,t_k)^{-1}.
\end{align*}
The proof follows directly by Lemma~\ref{lem:green} together with the fact that 
\begin{align*}
f=f_\mu,\qquad  \mu=\sum_{k=1}^K\alpha_k\beta(n_k,t_k)^{-1}\delta_{(n_k,t_k)},
\end{align*}
and
\begin{align*}
\left\|\frac{\mu+(-1)^m\mu^\vee}{2}\right\|_{\TV}
% &=\left\|\sum_{k=1}^K\alpha_k\beta(n_k,t_k)^{-1}\frac{\delta_{(n_k,t_k)}+(-1)^m\delta_{(-n_k,-t_k)}}{2}\right\|_{\TV}\\
&=\sum_{k=1}^K |\al_k|\beta(n_k,t_k)^{-1}.
\end{align*}
\end{rmk}

In the next section we provide an alternative construction of RKBS for ReLU type neural networks
where the polynomial space $\mathcal P_m$ is avoided.

\subsubsection{An alternative construction} \label{rmk:section}
As $ \Th = \P^d $, the space of all hyperplanes in $\R^d$, which is the natural
domain of the Radon transform.
For every hyperplane
$\xi\in \P^d$ there exists $(n,t)\in \Xi$ such that 
\[
x\in\xi \Longleftrightarrow x\cdot n=t .
\]
see~\Cref{fig:1}.
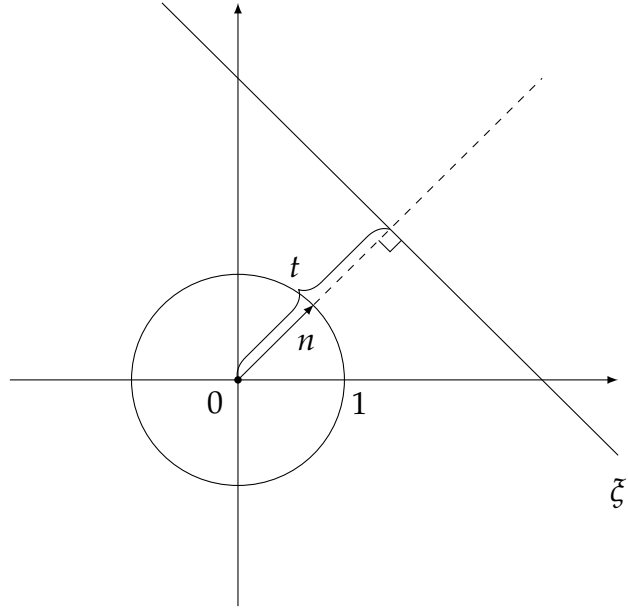
\begin{figure}[ht]
\centering
\begin{tikzpicture} [>=latex]
\draw [decorate,decoration={brace,amplitude=8pt}] (0,0) -- (2,2);
\node at (0.75,1.5) {$t$};
\draw plot [domain=-1:5] (\x, {-\x+4)}); %hyperplane 
%\draw plot [domain=-1:4] (\x, {\x)}); 
\draw (0,0) circle (1.4 cm);
%\draw [dashed] (-1,-1) - - (0,0); %perpendicular hyperplane
\draw [dashed] (1,1) - - (4,4); %perpendicular hyperplane
%\node at (1.5,4.6) {$\pi_0$}; %affine chart \pi_0
\draw [->] (-3,0) -- (5,0); %axis x
%\node at (5,-0.2) {$x$}; %axis x
\draw [->] (0,-3) -- (0,5); %axis y
%\node at (-0.3,4.6) {$y$}; %axis y
\fill (0,0) circle (1.4pt); %origin
%\node at (-0.2,-0.2) {0}; %origin
\node at (0.9,0.5) {$n$}; %perpendicular vector n(v)
\draw [->] (0,0) -- (1,1);
\node at (5,-1.5) {$\xi$}; %hyperplane equation
\node at (1.6,-0.3) {$1$}; %intersection point (t,0)
\node at (-0.3,-0.3) {$0$}; %intersection point (t,0)
%\fill (4,0) circle (1.4pt); %intersection point (t,0)
%\draw [line width=5pt] (1.87,1.87) -- (2,2);
\draw (2,2) coordinate (b);%angle
\draw[anchor=base,color=black]  (b.center) ++(-0.15,-0.15)  -- ++(0.15,-0.15) -- ++(0.15,0.15);%angle
\end{tikzpicture}
\caption{The hyperplane $\xi$ with equation $n\cdot x=t$
  (two-dimensional case). \label{fig:1}}
\label{fig:radon}
\end{figure}
The space $\Xi$ is a double cover of $\P^d$ with covering map\footnote{A double cover of a topological space $X$
is a topological space $C$
together with a continuous surjective map
$ p : C \to X $, called covering map,
such that, for every $ x \in X $,
there exists an open neighborhood $U$ of $x$
such that $p^{-1}(U)$ is the union of two disjoint open sets in $C$,
each of which homeomorphic to $U$ via $p$.}
\[
\Psi\colon \Xi\to \P^d,\qquad \Psi(n,t)=\{x\in\R^d: x\cdot n=t\} ,
\]
and $ \Psi(n,t)=\Psi(n',t') $ if and only if  $(n',t')=(-n,-t)$. 
Therefore, we can identify $\P^d$ with the quotient space $\Xi/\sim$,
where $\sim$ is the equivalence relation on $\Xi$ given by 
  \begin{equation}
(n,t)\sim (n',t') \Longleftrightarrow (n',t') = (-n,-t) .\label{eq:22}
\end{equation}
We denote by $ [(n,t)] \in \P^d $ the equivalence class of $(n,t)\in\Xi$.
Note that $\rho_m$ given in \eqref{eq:rho_m} is not well-defined on $\P^d$ since $\rho_m(x,
n,t)\neq \rho_m(x,-n,-t)$. To overcome this problem, we fix a
measurable section
\[
s\colon \P^d\to \Xi ,
\qquad
s(\xi)=(n(\xi),t(\xi)),
\]
{\it i.e.} $s$ is a measurable map satisfying 
$$
\xi=[s(\xi)],
$$
for every $\xi\in\P^d$. Then, we define the feature map 
\[
\widetilde{\phi}_m : \R^d \to C_0(\P^d)\subset\mm(\P^d)'
\]
given, for every $x \in \R^d$ and $\xi\in\P^d$, by
\begin{align*}
\widetilde{\phi}_m(x)(\xi) &= \sigma_m( n(\xi) \cdot x - t(\xi))\beta(n(\xi),t(\xi)),
\end{align*}
where the smoothing function  $\beta$ satisfies~\eqref{eq:2} and is
strictly positive. Further, we suppose $\beta$ to be an even function if $m$ is even and an odd function if $m$ is odd. This last assumption ensures that the right-hand side in formula~\eqref{eq:greensection} has the right parity (cf. Remark~\ref{rmk:parity}). We thus define the RKBS $\widetilde{\bb}_m$ as the RKBS associated with the feature map $\widetilde{\phi}_m$ according to \Cref{prop:RKBS}.
As we will see,
a crucial point to characterize the norm of $\bb_m$ lies in \Cref{lem:green}.
For the corresponding characterization in the space $\widetilde{B}_m$, one can prove an alternative version of Lemma~\ref{lem:green}.
\begin{lem} \label{lem:greensection}
  For every $ f_\mu \in \widetilde{\bb}_m$,
  \begin{equation}
    \label{eq:greensection}
   \frac{1}{2(2\pi)^{d-1}} \partial_t^{m}
 \Lambda^{d-1}\mathcal{R} f_\mu = \beta\mu,
  \end{equation}
  where the equality holds in $\mathcal{S}_0'(\Xi)$.
% \begin{equation}\label{eq:green1}
%  \frac{1}{\red{2(2\pi)^{d-1} }} \partial_t^{m}
%  \Lambda^{d-1}\mathcal{R} f_\mu = \beta\ \mu_{\rm even} \quad
%  \text{in\quad $\cS_0'(\Xi)_{\rm even}  $}. 
% \end{equation}
% Let $m> 2$ be an odd integer. For every $ f_\mu \in \bb_m$,
% \begin{equation}\label{eq:green2}
%  \frac{1}{\red{2(2\pi)^{d-1} }}  \partial_t^{m} \Lambda^{d-1}\mathcal{R} f_\mu = \beta\ \mu_{\rm even} \quad \text{in\quad $\cS_0'(\Xi)_{\rm odd} $}.
% \end{equation}
\end{lem}

We skip the proof of Lemma~\ref{lem:greensection} since it is similar to the the proof of Lemma~\ref{lem:green}. Then, one can prove the following result.

\begin{cor} \label{cor:alternative}
The problem   
\begin{equation*} %\label{eq:problemReLU1}
  \inf_{f\in \widetilde{\bb}_m}\left(   \frac{1}{N} \sum_{i=1}^N L ( y_i , f(x_i) ) + \|f\|_{\widetilde{\bb}_m} \right)
  \end{equation*}
 always has minimizers of the form
  \begin{equation*} %\label{eq:solutionsReLU1}
    f(x) = \sum_{k=1}^K \al_k \sigma_m ( n_k \cdot x - t_k ),
  \end{equation*}
  where $ K \le N $, $(n_k,t_k)\in S^{d-1}\times \R$, $\al_k\in\R\setminus\{0\}$
  and
  \[\|f\|_{\widetilde{\bb}_m} = \sum_{k=1}^K |\al_k|\beta(n_k,t_k)^{-1} .\]
 Furthermore, the map $ \mu \mapsto f_\mu $
 is an isometry from $ \mm(\P^d) $ onto $\widetilde{\bb}_m$, and
 $$
\|f_{\mu}\|_{\widetilde{\bb}_m}=\|\mu\|_{\TV}=\left\|  \frac{1}{2(2\pi)^{d-1}\beta}\partial_t^{m} \Lambda^{d-1}\mathcal{R} f_\mu\right\|_{\TV},\qquad \mu\in\mm(\P^d).
$$
\end{cor}
The last part of the statement follows by showing that the map $ \mu \mapsto f_\mu $ is injective and by \eqref{eq:normrkbsrelu}. The injectivity of the map is a consequence of Lemma~\ref{lem:greensection} together with the fact that $\beta\mu=0$ in $\mathcal{S}_0'(\Xi)$ implies $\mu=0$ in $\mathcal{M}(\P^d)$, see Lemma~\ref{embedding}. In other words, taking the feature map with values in $\mm(\P^d)'$ avoids the redundant parametrization of the RKBS caused by the odd measures.
\begin{rmk}
The introduction of the section $s$ is technically crucial. A natural alternative to make the feature map well defined on $\P^d$ is to symmetrize
the feature map, {\it i.e.} 

\[
\widetilde{\phi}_m(x)(\xi) = \frac{\sigma_m( n \cdot x - t)\beta(n,t)+ \sigma_m(
  -n \cdot x + t)\beta(-n,-t)}{2},\qquad \xi=[(n,t)].
\]

However, this would result in a representation with symmetrized
activation functions. For instance, for $m=2$ we would obtain neural
networks with absolute value activation function instead of the ReLU,
{\it i.e.}
\begin{equation*}
  f(x) = \sum_{k=1}^K \al_k |n_k \cdot x - t_k | ,
\end{equation*}
since 
\[
\sigma_m(n \cdot x - t)+ \sigma_m( -n \cdot x + t) =|n \cdot x - t| .
  \]
This is roughly the strategy followed in \cite{parhi2021banach}, where
the authors obtain representations with symmetrized activation
functions, but with an additional polynomial term (see \cite[Definition~5 with Remarks 6 and 7]{parhi2021banach}).
Note that
\[
\sigma_m(n \cdot x - t)- \sigma_m( -n \cdot x + t) = -n \cdot x + t ,
\]
which is a polinomial of degree 1 in $x$. In this view, the use of the
section $s$ provides a more transparent construction.
\end{rmk}

\begin{rmk}
In \Cref{cor:alternative} we obtain the same representation as in \Cref{cor:relurep},
but with a simplified regularization compared to \Cref{main2}. Moreover,
the norm of a solution $f_\mu$ is equal to (and not only controlled by) the $\ell^1$ norm of the representation coefficients.
\end{rmk}

%$\operatorname{dim}\mathcal{K}<+\infty$, and for every $\nu\in \mm(\Xi)_{\rm odd}$
%$$
%\|\nu\|_{\TV}=\inf_{\nu_0\in \operatorname{Ker}\Phi}\|P_{\mathcal{K}}\nu+\nu_0\|_{\TV}.
%$$

\subsection{Discussion: a comparison with previous results}
\label{sec:discussion-3}

In \cite{parhi2021banach} the authors build a family of function spaces $\mathcal{F}_m$, and seminorms $\phi_m\colon \mathcal{F}_m\to\R_+$ in terms of the Radon transform, such that the minimization problem    
\begin{equation*}
  \inf_{f\in \mathcal{F}_m}\left(   \frac{1}{N} \sum_{i=1}^N L ( y_i , f(x_i) ) + \phi_m(f) \right)
  \end{equation*}
 always has minimizers of the form
  \begin{equation}\label{eq:solutionparhinowak}
    f(x) = \sum_{k=1}^K \al_k (\sigma_m ( n_k \cdot x - t_k )+(-1)^m\sigma_m ( -n_k \cdot x + t_k ))+p(x),
  \end{equation}
  where $ K \le N $, $(n_k,t_k)\in S^{d-1}\times \R$, $\al_k\in\R\setminus\{0\}$
  and $p$ is a polynomial of order less than $m$. We refer to Theorem 1 in \cite{parhi2021banach} for the precise statement. If we compare equations~\eqref{eq:solutionsReLU} and \eqref{eq:solutionparhinowak}, we can highlight our two main contributions. The first one consists in getting rid of the polynomial term by considering a norm, instead of a seminorm, as regularization term. A second issue that we are able to overcome with our approach is to avoid solutions with symmetrized activation functions as in \eqref{eq:solutionparhinowak}. In particular, we choose the feature map with values either in $\mathcal{M}(S^{d-1}\times\mathbb{R})'$, or in $\mathcal{M}(\mathbb{P}^d)'$ but pre-composing the feature map with a measurable section $s\colon\P^d\to S^{d-1}\times\mathbb{R}$ (see \Cref{rmk:section} for full details). In view of Theorem~\ref{thm:representer}, we first define the hypothesis space as a RKBS. Then, we show an alternative approach to rigorously characterize the regularization term, and consequently the hypothesis space, in terms of the Radon transform, which is the content of Theorem~\ref{main2}. Conversely, in \cite{parhi2021banach} the authors start building {\it ad hoc} a family of seminorms in terms of the Radon transform, and consequently a family of hypothesis spaces. Their construction is motivated by Lemma~\ref{lem:green}. Then, in a second moment, they show the Banach space structure of the hypothesis spaces. A limitation of the approach in \cite{parhi2021banach} is that from their construction it is not evident how to identify new hypothesis spaces for other types of activation functions. In our approach, the identification of the hypothesis space follows straightforwardly by Theorem~\ref{thm:representer}, and it is independent of the relation between the Radon transform and the truncated power activation functions. 
  %Then, we exploit the latter relation, i.e. Lemma~\ref{lem:green}, to characterize the norm of the hypothesis space, see equation \eqref{eq:33}. 
  Finally, it is worth observing that our approach provides an integral representation for all the elements of the hypothesis space. This latter result is achieved by introducing the smoothing regularizer $\beta$, that ensures the convergence of the integral~\eqref{eq:17} without modifying the desired form for the minimizers~\eqref{eq:solutionsReLU}. In previous works, where $\beta$ is not introduced, the authors need to require alternative assumptions, as discussed in Remark~\ref{rmk:beta}.

\section{Radon transform: review and extension}\label{sec:radon}

% \subsection{Radon transform of test functions and distributions}\label{sec:radon}
We start recalling the function spaces that will come into play. Let $d\in\N$, $d\geq 1$. We use the notation $\langle x \rangle=(1+|x|^2)^{\frac{1}{2}}$. We denote by $\mathcal{S}(\mathbb{R}^{d})$ the Schwartz space of rapidly decreasing functions. We recall that a function $\varphi\colon\R^d\to\C$ belongs to $\mathcal{S}(\mathbb{R}^{d})$ if $\varphi\in C^{\infty}(\R^d)$ and 
\begin{equation}\label{eq:43bis}
\rho_{m,\alpha} (\varphi ) = 
\sup_{x \in \R^d} \langle x \rangle^m |\partial^{\alpha} \varphi(x) |<+\infty,
\qquad \forall m, \alpha \in \N^d .
\end{equation}
We endow $\mathcal{S}(\mathbb{R}^{d})$ with the topology induced by the family of seminorms $\{\rho_{m,\alpha}\}_{m,\alpha\in\N^d}$,
which makes $\mathcal{S}(\mathbb{R}^{d})$ a Fr\'echet space. Its dual space $\mathcal{S}'(\mathbb{R}^{d})$ is known as the space of tempered distributions. We use the notation $\mathcal{P}(\R^d)$ for the space of all polynomials on $\R^d$ and we denote by $\mathcal{S}_0 (\mathbb{R}^{d})$ the space of functions in $\mathcal{S}(\mathbb{R}^{d})$ that are orthogonal to all 
polynomials, {\it i.e.} 
\begin{equation}\label{eq:25}
  \mathcal{S}_0 (\mathbb{R}^{d}) = \left\{\varphi\in
\mathcal{S}(\mathbb{R}^{d}): \: \int_{\R^d} \varphi(x) p(x) {\rm d}x= 0,\ \forall p\in\mathcal{P}(\R^d)\right\}.
\end{equation}
The space $\mathcal{S}_0 (\mathbb{R}^{d})$ is called the Lizorkin test function space. It is a closed subspace of $\mathcal{S}(\mathbb{R}^{d})$ and we endow it with the relative topology inherited from $\mathcal{S}(\mathbb{R}^{d})$. Its dual space $\cS_0'(\R^d)$ of Lizorkin distributions is topologically isomorphic to the quotient space $\mathcal{S}'(\mathbb{R}^{d}) / \mathcal{P}(\R^d)$, see {\it e.g.} \cite[Chapter 1, Section 25]{hol95}.
\begin{lem}[{\cite[Lemma 6.0.4]{hol95}}]\label{lem:vanishingmomentsequivalent}
Let $\varphi\in\mathcal{S}(\R)$. Then $\varphi\in\mathcal{S}_0(\R)$ if and only if, for every $k\in\N$,
\begin{equation*}
\lim_{\omega\to0}\frac{\cF\varphi(\omega)}{|\omega|^k}=0.
\end{equation*}
\end{lem}
As a consequence of~\Cref{lem:vanishingmomentsequivalent}, the Fourier transform maps $\mathcal{S}_0(\R)$ into the space $\hat{\mathcal{S}}_0(\R)$ of rapidly decreasing functions that vanish in zero together with all of their partial derivatives, {\it i.e.}
\begin{equation*}
\hat{\mathcal{S}}_0 (\mathbb{R})
= \left\{\varphi\in \mathcal{S}(\mathbb{R}): \: \partial ^{m} \varphi (0) = 0,\ \forall m \in \mathbb{N} \right\}.
\end{equation*}
Recall that  $\Xi=S^{d-1}\times \R$. In analogy with $\cS(\R^d)$, we denote by $\cS(\Xi)$ the space of functions in $C^\infty(\Xi)$ such that
\[
\rho_{k,l,D}(\psi)=\sup_{n\in S^{d-1}, t\in\R} \langle t\rangle^k\left|\frac{{\rm d}^l}{{\rm d} t^l}D \psi (n,t)\right|<+\infty,
\]
for every $k,l\in\N$ and for every differentiable operator $D$ on
$S^{d-1}$. We endow  $\cS(\Xi)$ with the topology induced by the
family of seminorms $\rho_{k,l,D}$, and we denote by $\cS'(\Xi)$ its
topological dual space.  In analogy with the Lizorkin test function
space, $\cS_0(\Xi)$ denotes the set of functions $\psi\in\cS(\Xi)$
such that  
\begin{equation}\label{eq:vanisnhingmomentscondition}
\int_{\R}\psi(n,t) p(t) {\rm d}t=0,\qquad \forall p\in\mathcal{P}(\R), n\in S^{d-1}.
\end{equation}
Note that the integrals in \eqref{eq:vanisnhingmomentscondition} are finite since the functions $t\mapsto t^k\psi(n,t)$ belong to $L^1(\R)$ for every $k\in\N$ and $n\in S^{d-1}$. Then, by~\Cref{lem:vanishingmomentsequivalent}, condition \eqref{eq:vanisnhingmomentscondition} is equivalent to requiring that
\begin{equation*}%\label{eq:vanisnhingmomentscondition2}
\lim_{\omega\to0}\frac{\cF\psi(n,\omega)}{|\omega|^k}=0,
\qquad \forall k\in\N , n\in S^{d-1} ,
\end{equation*}
where $\mathcal{F}$ denotes the Fourier transform acting on the second
variable. We further refer to \cite{helgason99} for a complete exposition of the
function spaces introduced above.

  \begin{rmk}\label{strange}
    Usually, the Radon transform $\rr f$ of a function $f:\R^d\to\C$
    is defined on the space $\P^d$ of all hyperplanes in $\R^d$. As
    seen in~\Cref{rmk:section}, $\Xi$ is the double covering of $\P^d$
    with respect to the equivalence relation~\eqref{eq:22}. Hence, we
    can identify functions and distributions on $\P^d$ with even
    functions and even distributions on $\Xi$ and we can define the
    distribution Radon transform as a map from $\cS_0'(\R^d)$ onto
    $\cS_0'(\Xi)_{\rm even}\simeq \cS_0'(\P^d)$ and its dual $\rr^*$
    as a map from $\cS_0'(\Xi)_{\rm even}\simeq \cS_0'(\P^d)$ into
    $\cS_0'(\R^d)$.  We adopt this setting since the space
    $\cS_0'(\P^d)$ is replaced by $\cS_0'(\Xi)_{\rm odd}$ to deal with
    odd $m$, see~\Cref{main2}.
  \end{rmk}
We briefly recall the notion of even and odd distributions. For all functions $ \psi:\Xi\to\C$,  we set
\[
\psi^\vee:\Xi\to\C,\qquad \psi^\vee(n,t) = \psi(-n,-t), \qquad (n,t)\in \Xi .
  \]
  It is easy to check that
  \[
    \cS_0(\Xi)\ni \psi \mapsto \psi^\vee \in  \cS_0(\Xi)
    \]
is a well-defined continuous involution and, by duality,  it defines an
involution on $\cS_0'(\Xi)$
\[
    \cS_0'(\Xi)\ni g \mapsto g^\vee \in  \cS_0'(\Xi).
    \]
We set
\begin{align*}
    & \cS_0(\Xi)_{\rm even} =\{ \psi\in \cS_0(\Xi)  : \psi^\vee=\psi\},
  && \cS_0(\Xi)_{\rm odd} =\{ \psi\in \cS_0(\Xi)  : \psi^\vee=-\psi \},\\ 
    & \cS_0'(\Xi)_{\rm even} =\{ g\in \cS_0'(\Xi)  : g^\vee=g\},  &&\cS_0'(\Xi)_{\rm odd} =\{ g\in \cS_0'(\Xi) : g^\vee=- g \} ,
  \end{align*}
  which are closed subsets of  $\cS_0(\Xi)$ and $\cS_0'(\Xi)$,
  respectively. Moreover,
    \begin{align*}
  & \cS_0(\Xi)  =\cS_0(\Xi)_{\rm even}+ \cS_0(\Xi)_{\rm odd},
      &&\cS_0(\Xi)_{\rm even}\cap \cS_0(\Xi)_{\rm odd} =\{0\}, \\
 & \cS_0'(\Xi)  =\cS_0'(\Xi)_{\rm even}+ \cS_0'(\Xi)_{\rm odd},
      &&\cS_0'(\Xi)_{\rm even}\cap \cS_0'(\Xi)_{\rm odd} =\{0\} ,
    \end{align*}
    where the maps
    \begin{align*}
      & \cS_0(\Xi)_{\rm even}\times  \cS_0(\Xi)_{\rm odd}\ni
        (\psi_{\rm even} ,\psi_{\rm odd})\mapsto \psi_{\rm even}
        +\psi_{\rm odd} \in \cS_0(\Xi),\\
      & \cS_0'(\Xi)_{\rm even}\times  \cS_0'(\Xi)_{\rm odd}\ni
        (g_{\rm even} ,g_{\rm odd})\mapsto g_{\rm even}
        +g_{\rm odd} \in \cS_0'(\Xi)
    \end{align*}
    are topological isomorphisms. A simple calculation shows that
    \begin{align*}
      &  \cS_0'(\Xi)_{\rm even} \simeq \left(\cS_0(\Xi)_{\rm even}\right)', \\
      &  \cS_0(\Xi)'_{\rm odd} \simeq \left(\cS_0(\Xi)_{\rm odd}\right)', 
    \end{align*}
which implies that $\cS_0'(\Xi)_{\rm even}\simeq \cS_0'(\P^d)$ under
the identification $\cS_0(\Xi)_{\rm even}=\cS_0(\P^d)$, as claimed
in~\Cref{strange}.  

With this setting, we can recall the definition of the Radon transform and its dual.
\begin{dfn} \label{defn:classicalradon}
The Radon transform of $\varphi\in L^1(\R^d)$ is the function $\mathcal{R}\varphi\colon \Xi\to \C$ defined by 
\begin{equation*} %\label{eq:dfnradon}
\mathcal{R}\varphi(n,t)=\int_{n\cdot x=t}\varphi(x)\D m(x),
\qquad \text{ for a.e. } (n,t) \in \Xi ,
\end{equation*}
where $m$ is the Euclidean measure on the hyperplane with equation $n\cdot x=t$.
\end{dfn}
Since the pairs $(n,t)$ and $(-n,-t)$ define the same
hyperplane, clearly the Radon transform is an even function, {\it
  i.e.} 
\begin{equation}\label{eq:symmetrycondition}
(\mathcal{R}\varphi)^\vee = \mathcal{R}\varphi.
\end{equation}

\begin{thm}[{\cite[Corollary 4.2]{helgason65}}]\label{thm:continuitybijectionradon}
The Radon transform is a continuous injective operator from
$\mathcal{S}_0(\R^d)$ onto $\mathcal{S}_0(\Xi)_{\rm even}$. 
\end{thm}
%In order to state the most commonly used inversion formula for the Radon %transform, known as Filtered Back Projection, 

We now introduce the dual Radon transform, also known as back-projection. While the Radon transform is defined for any pair $(n,t)$ as the integral over the set of points belonging to the hyperplane with equation $n\cdot x=t$, the dual Radon transform is defined for any given point $x\in\R^d$ as the integral over the set of hyperplanes passing through $x$, which corresponds to the set of pairs $\{(n,n\cdot x):n\in S^{d-1}\}\subseteq \Xi$. 
\begin{dfn}\label{defn:classicaldualradon}
The dual Radon transform (or back-projection) of $ \psi \in
L^\infty(\Xi) $ is the $L^\infty$ function $\rr^* \psi: \R^d\to\C$
defined by  
\begin{equation*}%\label{eq:classicaldualradon}
\rr^* \psi(x)=\int_{S^{d-1}} \psi(n,n\cdot x) \D n,\qquad x\in\R^d,
\end{equation*}
where $\D n$ is the spherical measure on $S^{d-1}$. 
\end{dfn}
Note that, if $\psi$ is an odd function, clearly  $\rr^*\psi=0$ since
$\D n$ is invariant under reflection. 
\begin{thm}[{\cite[Corollary 4.2]{helgason65}}]\label{thm:continuitybijectiondualradon}
The dual Radon transform is a continuous injective operator from
$\mathcal{S}_0(\Xi)_{\rm even}$ onto $\mathcal{S}_0(\R^d)$.
\end{thm}

We refer to \cite[Corollary 6.1]{kpsv2014} for an alternative proof of the continuity of the operators $\mathcal{R}\colon \mathcal{S}_0(\R^d)\to \mathcal{S}_0(\Xi)_{\rm even}$ and $\mathcal{R}^*\colon \mathcal{S}_0(\Xi)_{\rm even}\to \mathcal{S}_0(\R^d)$ based on the relation existing between Radon, ridgelet and wavelet transforms.

\begin{prop}[{\cite[Chapter II]{natterer}}]
For every $\varphi\in L^1(\R^d)$ and $\psi\in L^\infty(\Xi)$,
\begin{equation}\label{eq:dualityrelation}
\int_{\R^d} \varphi(x)\rr^* \psi(x)\D x=\int_{\Xi}\rr \varphi(n,t)
\psi(n,t)\D n\D t.
\end{equation}
\end{prop}
The duality relation \eqref{eq:dualityrelation} can be exploited to extend $\rr$ and $\rr^*$ on distribution spaces \cite{helgason99, hertle84, kpsv2014}.
\begin{dfn}\label{defn:radonanddualdistributions}
The Radon transform of $ f \in \cS_0'(\R^d) $ is the continuous linear functional $ \rr f $ on $\mathcal{S}_0(\Xi)_{\rm even}$ defined by
$$
\langle\rr f,\psi\rangle=\langle f,\rr^{*} \psi\rangle ,
\qquad  \psi\in\mathcal{S}_0(\Xi)_{\rm even}.
$$
Analogously, the dual Radon transform of $ g \in \mathcal{S}_0'(\Xi)_{\rm even}$ is the continuous linear functional on $ \cS_0(\R^d) $ defined by 
$$
\langle\rr^{*}g,\varphi\rangle=\langle g,\rr \varphi\rangle ,
\qquad \varphi\in\cS_0(\R^d) .
$$
\end{dfn}
Note that $ \rr : \cS_0'(\R^d) \to \mathcal{S}_0'(\Xi)_{\rm even} $ and $ \rr^* : \mathcal{S}_0'(\Xi)_{\rm even} \to \cS_0'(\R^d) $ are well defined and weakly continuous
thanks to~\Cref{thm:continuitybijectiondualradon} and \Cref{thm:continuitybijectionradon}, respectively.

We next recall the most commonly used inversion formula for the Radon transform, known as Filtered Back Projection.
To state the formula, we first need to introduce the positive symmetric operator
$ \La^{d-1} : \cS(\Xi) \to C^\infty(\Xi) $ defined by
  \begin{equation}
\Lambda^{d-1} \psi(n,t)=
\begin{cases}
(-1)^{\frac{d-1}{2}}\partial_t^{d-1}\psi(n,t) & \text{$d$ odd}\\
(-1)^{\frac{d-2}{2}}\mathscr{H}\partial_t^{d-1}\psi(n,t) & \text{$d$ even}
\end{cases} ,\label{eq:27}
\end{equation}
where the Hilbert transform $\mathscr{H}$ acts only on the second variable. The operator $\La^{d-1}$ is also known as ramp filter. 
%\begin{equation}\label{eq:operatorinversionformula}
%\cF[\Lambda^{d-1} \psi](n,\omega)=|\omega|^{d-1}\cF \psi(n,\omega),
%\end{equation}
%and $\cF$ denotes the Fourier transform acting on the second variable. 

\begin{thm}[ {\cite[Chapter I, Theorems 3.6 and 3.5]{helgason99}} ]
%\cite[Theorem 2.1]{natterer}
\label{teo:bagkprojectionformulaclassical}
For every $\varphi\in \cS(\R^d)$,
\begin{equation}\label{bagkprojectionformulaclassical}
\varphi=\frac{1}{2(2\pi)^{d-1}}\rr^*\Lambda^{d-1}\rr \varphi .
\end{equation}
For every $g\in\mathcal{S}_0(\Xi)_{\rm even}$,
\begin{equation}\label{bagkprojectionformulaclassical2}
g=\frac{1}{2(2\pi)^{d-1}}\Lambda^{d-1}\rr \rr^*g .
\end{equation}
\end{thm}

In \cite[Proposition~4.3]{hertle84}, the inversion formula \eqref{bagkprojectionformulaclassical} has been extended to the space $\mathcal{D}'_{L^1}(\R^d)$ of Schwartz integrable distributions \cite{schwartz66}, which embeds densely in $\mathcal{S}_0'(\R^d)$.
We will now provide extensions of \eqref{bagkprojectionformulaclassical} and \eqref{bagkprojectionformulaclassical2} to $\mathcal{S}_0'(\R^d)$ and $\mathcal{S}_0'(\Xi)_{\rm even}$, respectively.  

It is worth observing that the Hilbert transform appears in the expression of the operator $\Lambda^{d-1}$ only when the dimension $d$ is even. This difference is crucial in the Radon transform theory. For odd dimension $d$, 
$\Lambda^{d-1}$ is a differential operator and it is therefore clear
that it maps $\cS(\Xi)$ continuously into itself. This no longer holds
if $d$ is even, because the Hilbert transform maps $\mathcal{S}(\R)$
into $C^\infty(\R)$, but not into $\mathcal{S}(\R)$ \cite{ludwig1966radon}. A more satisfactory situation is obtained if we restrict our attention to the smaller space of functions  $\mathcal{S}_0(\Xi)$.

\begin{lem} \label{lem:LaS_0}
 The Hilbert transform maps $\cS_0(\R)$ continuously into itself,
 and therefore $ \La^{d-1} $ maps $\mathcal{S}_0(\Xi)_{\rm even}$ continuously into itself for every $d\ge1$.
\end{lem}
\begin{proof}
We start showing that $\mathscr{H}$ maps $\mathcal{S}_0(\R)$ into $\mathcal{S}(\R)$.
Let  $\varphi\in\mathcal{S}_0(\R)$.
We already know that $\mathscr{H}\varphi\in C^{\infty}(\R)$.
Thus, it remains to show that $\mathscr{H}\varphi$ is a rapidly decreasing function,
or equivalently that $\cF[\mathscr{H}\varphi]\in \cS(\R)$.
We recall that $\mathscr{H}$ maps $\cS(\R)$ into $L^2(\R)$, and for every $\varphi\in\mathcal{S}(\R)$ it satisfies 
\begin{equation*}
\mathcal{F}[\mathscr{H}\varphi](\omega) = -i\, \text{sgn}(\omega)\mathcal{F}\varphi(\omega),\quad \text{for a.e. } \omega\in\R.
\end{equation*}
Hence, we have that $\cF[\mathscr{H}\varphi]\in C^\infty(\R\setminus\{0\})$,
and for every $l\in\N$
\begin{equation}\label{eq:derivativeHilberttransform}
\partial_\omega^l\cF[\mathscr{H}\varphi](\omega) =
-i \sgn(\omega) \partial_\omega^l\mathcal{F}\varphi(\omega), \qquad \omega\ne0 .
\end{equation}
Since $\varphi\in\mathcal{S}_0(\R)$,
$
\partial_\omega^l\mathcal{F}\varphi(0)=0
$ 
for every $l\in\N$, and $\cF[\mathscr{H}\varphi]$ can be extended together with all its derivatives to  continuous functions on $\R$. 
Therefore, $\cF[\mathscr{H}\varphi]\in C^\infty(\R)$ and hence $\mathscr{H}\varphi\in\mathcal{S}(\R)$. In fact, $\mathscr{H}\varphi\in\cS_0(\R)$. Indeed, since $\varphi\in\cS_0(\R)$, for every $k\in\N$
\begin{equation*}
\lim_{\omega\to0}\frac{\cF[\mathscr{H}\varphi](\omega)}{\omega^k}=
\lim_{\omega\to0}\frac{-i\, \sgn(\omega)\mathcal{F}\varphi(\omega)}{\omega^k}=0,
\end{equation*}
which implies $\mathscr{H}\varphi\in\cS_0(\R)$ by~\Cref{lem:vanishingmomentsequivalent}.
We now show that $\mathscr{H}$ is continuous from $\cS_0(\R)$ into itself.
In view of \eqref{eq:derivativeHilberttransform},
for every $\omega\in\R$ and $  m , \alpha \in \N $ we have
\begin{align*}
\langle\omega\rangle^m|\partial_\omega^\alpha\cF\mathscr{H}\varphi(\omega)|=\langle\omega\rangle^m|\partial_\omega^\alpha\cF\varphi(\omega)| .
\end{align*}
The claim follows by observing that
$\rho_{m,\alpha}(\cF \varphi)$, $m,\alpha\in\N$, defines a basis of seminorms for the topology of $\cS_0(\R)$.
Therefore, since $\cS_0(\R)$ is closed under differentiation and since $\mathscr{H}$ maps $\cS_0(\R)$ continuously into itself, it is clear from the definition that $\La^{d-1}$ maps $\cS_0(\Xi)$ continuously into itself for every $d\geq1$. Furthermore, if $g\in \mathcal{S}(\Xi)_{\rm even}$, then $\Lambda^{d-1} g$ satisfies the symmetry condition \eqref{eq:symmetrycondition} \cite[Chapter~I, Section~3]{helgason99}. Therefore, $\Lambda^{d-1}$ maps $\mathcal{S}_0(\Xi)_{\rm even}$ into itself for every $d\geq1$.
\end{proof}

Thanks to~\Cref{lem:LaS_0}, we can define the weakly continuous
operator $\Lambda^{d-1}\colon
\mathcal{S}'_0(\Xi)_{\rm even}\to\mathcal{S}'_0(\Xi)_{\rm even}$ given by 
\begin{equation} \label{eq:La-ext}
\langle\Lambda^{d-1}g, \varphi\rangle=\langle g,\Lambda^{d-1}\varphi\rangle ,
\qquad g\in \mathcal{S}'_0(\Xi)_{\rm even}, \varphi\in\mathcal{S}_0(\Xi)_{\rm even}.
\end{equation}
We are now able to extend the inversion formulae \eqref{bagkprojectionformulaclassical} and \eqref{bagkprojectionformulaclassical2} to Lizorkin distributions.
\begin{cor}\label{cor:inversionformuladistributions}
For every $f\in \mathcal{S}_0'(\R^d)$,
\begin{equation*}
f=\frac{1}{2(2\pi)^{d-1}}\rr^*\Lambda^{d-1}\rr f.
\end{equation*}
For every $g\in \mathcal{S}'_0(\Xi)_{\rm even}$,
\begin{equation*}
g=\frac{1}{2(2\pi)^{d-1}}\Lambda^{d-1}\rr \rr^*g.
\end{equation*}
\end{cor}
\begin{proof}
The proof follows combining inversion formulas \eqref{bagkprojectionformulaclassical} and \eqref{bagkprojectionformulaclassical2} together with~\Cref{defn:radonanddualdistributions} and equation \eqref{eq:La-ext}.
\end{proof}

\subsection{Discussion: our contribution in Radon inversion} \label{sec:discussion-4}
An important problem in harmonic analysis is the extension of a linear
operator from a Hilbert space to generalized function spaces. The
classical approach is to define the extended operator by
transposition. A standard example is the definition of the Fourier
transform on tempered distributions \cite{schwartz66}. The extension
of the Radon transform, and of the related inversion formulae, is a
well-known subject and it is deeply studied in \cite{hertle84,
  helgason99, kpsv2014}. In particular, in
\cite[Proposition~4.3]{hertle84} the author extends the inversion
formula \eqref{bagkprojectionformulaclassical} to the space of
Schwartz integrable distributions
$\mathcal{D}'_{L^1}(\R^d)\subseteq\mathcal{S}_0'(\R^d)$. Our
contribution consists in showing that the inversion formulae
\eqref{bagkprojectionformulaclassical} and
\eqref{bagkprojectionformulaclassical2} actually extend to the larger
spaces of Lizorkin distributions $\mathcal{S}_0'(\R^d)$ and
$\mathcal{S}_0'(\Xi)_{\rm even}$, a fact which we largely exploit in
Sections~\ref{sec:reluradon} and \ref{sec:proof-main}. More precisely,
Corollary~\ref{cor:inversionformuladistributions} follows directly by
Lemma~\ref{lem:LaS_0}, which allows to extend the Hilbert transform,
and consequently the operator $\Lambda^{d-1}$, to Lizorkin
distributions. To the best of our knowledge, Lemma~\ref{lem:LaS_0}
does not appear in the literature and, together with
Corollary~\ref{cor:inversionformuladistributions}, contributes to
enrich the distributional framework for the Radon transform.

\section{Proofs of \Cref{sec:main2}} \label{sec:proof-main}

We provide a detailed analysis of the main results of 
\Cref{sec:main2}.
We will make use of the classical function and distribution spaces
listed in \Cref{tab:function_spaces},
on the domains listed in \Cref{tab:spaces}.
In \Cref{tab:operators} we recall the main linear operators involved.
For definitions and
properties we refer to \Cref{sec:radon}.

\begin{table}[ht]
%\captionsetup{labelformat=empty}
\caption{Domains ($S^{d-1}$ denotes the unit sphere in $\R^d$).}
\label{tab:spaces}
\centering
\ra{1.25}
%\resizebox{\columnwidth}{!}{
\begin{tabular}{l l}

\hline

$ \R^d $ & input space \\
  
%$ \P^d$ &  space of hyperplanes in $\R^d$\\

$ \Xi=S^{d-1} \times \R $ & parameter space  \\
\hline

\end{tabular}
%}
\end{table}

\begin{table}[ht]
%\captionsetup{labelformat=empty}
\caption{Function and distribution spaces ($ X = \R^d , \Xi $).\\
Subscripts $(\Xi)_{\rm even}$ and $ (\Xi)_{\rm odd} $ denote the corresponding subspaces of even and odd measures/functions/distributions, respectively.
}
\label{tab:function_spaces}
\centering
\ra{1.25}
%\resizebox{\columnwidth}{!}{
\begin{tabular}{l l}

\hline

$ \mm(X) $ & real bounded  measures on $X$  \\

$ \cS(X) $ & Schwartz space of rapidly decreasing functions on $X$ \\

$ \cS'(X) $ & tempered distributions on $X$ \\

$ \cS_0(X) $ & Lizorkin test functions on $X$ \\

$ \cS_0'(X) $ & Lizorkin distributions on $X$ \\

\hline
\end{tabular}
%}
\end{table}

\begin{table}[ht]
%\captionsetup{labelformat=empty}
\caption{Operators.}
\label{tab:operators}
\centering
\ra{1.25}
%\resizebox{\columnwidth}{!}{
\begin{tabular}{l l}

\hline

Radon transform

&

$$
\xymatrix{
    \cS_0(\R^d) \ar@<1ex>[r]^{\rr} \ar@{^{(}->}[d] &
                                                     \cS_0(\Xi)_{\rm even} \ar@<1ex>[l]^{\rr^*} \ar@{^{(}->}[d] \\ 
    \cS_0'(\R^d) \ar@<1ex>[r]^{\rr} & \cS_0'(\Xi)_{\rm even} \ar@<1ex>[l]^{\rr^*}
    }
$$

\\

Ramp filter

& 

$$
\xymatrix{
    \cS_0(\Xi)_{\rm even} \ar[r]^{\La^{d-1}} \ar@{^{(}->}[d] & \cS_0(\Xi)_{\rm even} \ar@{^{(}->}[d] \\
    \cS_0'(\Xi)_{\rm even} \ar[r]^{\La^{d-1}} & \cS_0'(\Xi)_{\rm even}
    }
$$

\\

\hline

\end{tabular}
%}
\end{table}

The first lemma allows to regard $\bb_m$ as a subspace of the
space of tempered distributions.
We denote by $H:\R\to\R$ the Heaviside step function 
\[
H(t)=
\begin{cases}
0 & t<0\\
1 & t\geq 1
\end{cases},
\]
regarded as a temperated distribution.  
\begin{lem}\label{lem:growth}
  With the above notation,
  \begin{enumerate}[label=\textnormal{(\roman*)}]
  \item \label{it:growth1}
  $\sigma_m\in\cS'(\R)$ and
    \begin{equation} \label{eq:derho}
%\sigma_m'= \rho_{m-1},
% \qquad
 \sigma_m^{(m-1)} = H ,
% \qquad
%\sigma_m^{(m)} = \de,
\end{equation}
where the  equality holds true in  $\cS'(\R)$;
\item \label{it:growth2}
for all $(n,t)\in\Xi$, $\rho_m(\cdot, n,t)\in \cS'(\R^d)$;
\item\label{item:3} the elements $f\in \bb_m$ are continuous functions satisfying the
polynomial growth condition 
  \begin{equation}
    \label{eq:19}
    |f(x)| \leq C_f (1+|x|)^{m-1};
  \end{equation}
 \item \label{it:growth4}
 $\bb_m\subset \cS'(\R^d)$ .
  \end{enumerate}
\end{lem}
\begin{proof}
\ref{it:growth1} and \ref{it:growth2} are clear. We prove \ref{item:3}.
Let $ f \in \bb_m $.
By \eqref{eq:16}, there exists $\mu\in \bb_m$ such that
\[
f(x) = \int_{\Xi} \sigma_m( n \cdot x - t)\beta(n,t) \ \D\mu(n,t).
\]
Then, for every $x\in\R^d$,
\begin{align*}
|f_\mu(x)|&\leq \frac{1}{(m-1)!}\int_{\Xi}|\beta(n,t)| |n\cdot x-t|^{m-1} \ \D\mu(n,t)\\
&\leq\frac{1}{(m-1)!}\int_{\Xi}(|x|+|t|)^{m-1}|\beta(n,t)|\D\mu(n,t)\\
&=\frac{1}{(m-1)!}\sum_{k=0}^{m-1}\binom{m-1}{k}|x|^k \int_{\Xi}|t|^{m-1-k}|\beta(n,t)|\D\mu(n,t),
\end{align*}
where the integrals converge  by~\eqref{eq:20}. The right hand side is
a polynomial of degree less than $m$, hence we obtain \eqref{eq:19}.
We now prove that $f$ is continuos. Since
\[
f(x_0+h ) = \int_{\Xi} \sigma_m( n \cdot h + n \cdot x_0- t)\beta(n,t) \ \D\mu(n,t),
\]
it is enough to show that $f$ is continuos at $x_0=0$.
This is a consequence of the dominated  convergence
theorem, observing that,  for each $(n,t)\in \Xi$, $x \mapsto \sigma_m(
n \cdot x - t)\beta(n,t)$ is continuous and, by~\eqref{eq:21}, 
\[
\sup_{|x|\leq 1} |\sigma_m( n \cdot x- t)\beta(n,t) |\leq  (1+|t|)^m|\beta(n,t)|,
\]
where the right-hand side is integrable by~\eqref{eq:20}.~\Cref{it:growth4} is a direct consequence of \ref{item:3}.
\end{proof}

 The growth condition~\eqref{eq:19} is one starting point of the
 construction in~\cite{parhi2021banach} (see their equation~(8)). Note
 that, in our construction, the smoothing function $\beta$ allows us to prove 
that the elements of $\bb_m$ are continuous functions.

We need to introduce the following operator, which provides a bounded
inverse of the derivative. It was implicitly introduced in \cite{unser2017splines}.  
\begin{prop}\label{prop:operatorA}
The operator
\begin{equation*} %\label{eq:opD}
 \partial \colon  \mathcal{S}_{0} (\mathbb{R})\to
\mathcal{S}_{0} (\mathbb{R}) , \qquad \partial\psi(t)= \psi'(y) ,
\end{equation*}
is a continuous linear operator and, by duality, it extends to
a weakly continuous operator on  $\mathcal{S}_{0}' (\mathbb{R})$. 
The operator
\begin{equation*} %\label{eq:opA}
\mathcal{A}\colon  \mathcal{S}_{0} (\mathbb{R})\to
\mathcal{S}_{0} (\mathbb{R}) , \qquad \mathcal{A}\psi(t)=\int_{-\infty}^t\psi(s)\D s
\end{equation*}
is a continuous linear operator satisfying
\begin{equation}\label{eq:operatorA}
\mathcal{A}\partial\psi=\partial\mathcal{A}\psi=\psi , \qquad \psi \in \mathcal{S}_{0} (\mathbb{R}) .
\end{equation}
By duality, $\mathcal{A}$ extends to a weakly continuous operator on $\mathcal{S}_0'(\R)$
satisfying
\begin{equation}\label{eq:operatorAdistributions}
\mathcal{A}\partial f=\partial\mathcal{A}f=f, \qquad f\in \mathcal{S}_0'(\R) .
\end{equation}
\end{prop}
\begin{proof}
The first claim is a consequence of the fact that $\partial$ is a
continuous linear operator from $\cS(\R)$ to $\cS(\R)$ and
that the space of polynomials is stable under differentiation
(see \eqref{eq:25}).  Recall that $\langle x
\rangle=(1+|x|^2)^{\frac{1}{2}}$ and the family of
seminorms 
on $ \mathcal{S} (\R)$ is given by~\eqref{eq:43bis}. 
For $\varphi\in \mathcal{S}_{0} (\mathbb{R})$, we have
$$
\mathcal{A}\varphi(x) = \int_{-\infty}^x\varphi(t)\D t=-\int_{x}^{+\infty}\varphi(t)\D t .
$$
We show that $\mathcal{A}\varphi\in \mathcal{S}_{0} (\mathbb{R})$.  
For every $m\in\N$ and $x>0$, we have
\begin{align*}
\langle x \rangle^{m} |\mathcal{A}\varphi(x)|&=|\int_{x}^{+\infty}(1+x^2)^{\frac{m}{2}}\varphi(t)\D t|\leq\int_{x}^{+\infty}(1+t^2)^{\frac{m}{2}}|\varphi(t)|\D t\\
&\leq\rho_{2m+4,0}(\varphi)\int_{-\infty}^{+\infty}(1+t^2)^{\frac{m}{2}}\frac{1}{(1+t^2)^{m+2}}\D t<+\infty.
\end{align*}
Analogously, for every $m\in\N$ and $x<0$, we have
\begin{align*}
\langle x \rangle^{m} |\mathcal{A}\varphi(x)|&=|\int_{-\infty}^{x}(1+x^2)^{\frac{m}{2}}\varphi(t)\D t|\leq\int_{-\infty}^{x}(1+t^2)^{\frac{m}{2}}|\varphi(t)|\D t\\
&\leq\rho_{2m+4,0}(\varphi)\int_{-\infty}^{+\infty}(1+t^2)^{\frac{m}{2}}\frac{1}{(1+t^2)^{m+2}}\D t<+\infty.
\end{align*}
%%%%%%%%%%%%%
Thus, $\mathcal{A}\varphi$ is a well defined function, and for every $m\in\N$
\begin{equation}\label{eq:lemmaimportant1}
\sup_{x\in\R}\langle x \rangle^{m} |\mathcal{A}\varphi(x)| \le C \
\rho_{2m+4,0}(f)<+\infty
\end{equation}
for some positive constant $C$.
Furthermore, by definition, $\partial \mathcal{A}\varphi(x) = f(x),$ and thus, for every $m\in\N$
and $\alpha\geq 1$,
\begin{equation}\label{eq:lemmaimportant2}
\sup_{x\in\R}\langle x \rangle^{m} |\partial^\alpha \mathcal{A}\varphi(x)|=\sup_{x\in\R}\langle x \rangle^{m} |\partial^{(\alpha-1) }f(x)|<+\infty.
\end{equation}
Therefore, $\mathcal{A}\varphi\in\cS(\R)$.
Moreover, since $f\in\cS_0(\R)$, for every $n\in\N$ we have
\begin{equation*}
\int_{-\infty}^{+\infty}x^n\mathcal{A}\varphi(x)\D x=-\int_{-\infty}^{+\infty}x^{n+1}\partial\mathcal{A}\varphi(x)\D x=-\int_{-\infty}^{+\infty}x^{n+1}f(x)\D x=0.
\end{equation*}
Hence, $\mathcal{A}\varphi\in\cS_0(\R)$.
By \eqref{eq:lemmaimportant1} and \eqref{eq:lemmaimportant2} we have that, for every $m, \alpha\in\mathbb{N}$ and some constant  $C$,
\begin{align*}
\rho_{m,\alpha} (\mathcal{A}\varphi) = \sup_{x \in \R} \langle x \rangle ^m | \partial^\alpha \mathcal{A}\varphi(x)| \le C \ \rho_{2m+4,\alpha-1}(f),
\end{align*}
which shows that $\mathcal{A}\colon  \mathcal{S}_{0} (\mathbb{R})\to  \mathcal{S}_{0} (\mathbb{R})$ is continuous. \eqref{eq:operatorA}
is a direct consequence of the fundamental theorem of calculus. Since
$\aa$ is continuos, by duality $\aa$ extends to a weakly continuous
operator on $\mathcal{S}_{0}' (\mathbb{R})$  and~\eqref{eq:operatorAdistributions} follows directly from
\eqref{eq:operatorA}. 
\end{proof}
Note that the fact that  $\partial$ has a bounded inverse strongly
depends on the fact that its domain is $\cS_0(\R)$. 

The next proposition is at the root of \Cref{main2}.
It was first stated in \cite[Lemma 18]{parhi2021banach},
by using   the Radon transform $\rr$.
Here we provide an alternative proof based on
the dual Radon transform  $\rr^*$.    
\begin{prop}\label{prop:propReLUdirac}
For every $\varphi\in\mathcal{S}_0(\R^d)$ and for every $(n,t)\in\Xi$,
 \begin{align*}
{}_{\cS_0'(\R^d)}\langle\rho_m( \cdot, n ,t), \varphi \rangle_{\cS_0(\R^d)}
=(-1)^m \beta(n,t)\mathcal{A}^m(\mathcal{R}\varphi)(n,t),
 \end{align*} 
where $\mathcal{A}$ is the operator defined by~\eqref{eq:operatorA} acting  on
$\mathcal{R}\varphi$ as a function of the only second variable.  
\end{prop}
\begin{proof}
Let $\varphi\in\cS_0(\R^d)$. We can consider the function $T_\varphi\colon\Xi\to\mathbb{C}$ given by 
\[
T_\varphi(n,t)= \int_{\R^d} \sigma_m( x\cdot n -t)
\varphi(x)\ \D x .
\]
Reasoning as in the proof of~\Cref{item:3} of~\Cref{lem:growth}, it is
possible to show that $T_\varphi$ is a continuous function. 
We show that $T_\varphi\in \mathcal{S}_0'(\Xi)$. For every $(n,t)\in\Xi$,
\begin{align*}
|T_\varphi(n,t)| &\leq\int_{\R^d}|\sigma_m( n\cdot x- t)||\varphi(x)|\D x\\
&=\frac{1}{(m-1)!}\int_{\R^d}|n\cdot x- t|^{m-1}|\varphi(x)|\D x\\
&\leq\frac{1}{(m-1)!}\int_{\R^d}(|x|+|t|)^{m-1}|\varphi(x)|\D x\\
&=\frac{1}{(m-1)!}\sum_{k=0}^{m-1}\binom{m-1}{k}|t|^k \int_{\R^d}|x|^{m-1-k}|\varphi(x)|\D x,
\end{align*}
which is a polynomial of order $m-1$ in the $t$ variable. Now, we
compute the expression of the $m$- th derivative of $T_\varphi$ with respect to the variable $t$. Let $\psi\in \mathcal{S}_0(\Xi)$. Then
\begin{align*}
\langle \partial_t^m T_\varphi,\psi \rangle&=(-1)^m\langle T_\varphi,\partial_t^m\psi \rangle\\
%&=(-1)^m\int_\Xi \langle\rho_m( \cdot, n ,t), \varphi \rangle \partial_t^m\psi(n,t) \D n\D t\\
&=(-1)^m\int_\Xi \left(\int_{\R^d}\sigma_m( n\cdot x - t) \varphi(x) \D x\right) \partial_t^m\psi(n,t) \D n\D t\\
&=(-1)^m \int_{\R^d}\left(\int_{S^{d-1}}\int_\R\sigma_m( n\cdot x - t) \partial_t^m\psi(n,t) \D t\D n\right) \varphi(x) \D x.
\end{align*}
Hence, by \eqref{eq:derho},
\begin{align*}
\langle \partial_t^m T_\varphi,\psi \rangle&=(-1)^m\int_{\R^d}\left(\int_{S^{d-1}}\int_\R H( n\cdot x - t) \partial_t\psi(n,t) \D t\D n\right) \varphi(x) \D x\\
&=(-1)^m\int_{\R^d}\left(\int_{S^{d-1}}\int_{-\infty}^{n\cdot x}  \partial_t \psi(n,t) \D t\D n\right) \varphi(x) \D x\\
&=(-1)^m\int_{\R^d}\left(\int_{S^{d-1}}\psi(n,n\cdot x)\D n\right) \varphi(x) \D x .
%&=(-1)^m\int_{\R^d}\rr^*\psi(x)\ \varphi(x) \D x.
\end{align*}
If $\psi$ is an odd function, then $\int_{S^{d-1}}\psi(n,n\cdot x)\D
n=0$, so that $\langle \partial_t^m T_\varphi,\psi \rangle=0$. Hence
$\partial_t^m T_\varphi$ is an even distribution, {\it i.e.}
$\partial_t^m T_\varphi\in \cS_0'(\Xi)_{\rm even}$. If $\psi$ is an even
function, {\it i.e.} $\psi\in \cS_0(\Xi)_{\rm even}$, \Cref{defn:classicaldualradon} gives 
\[
\langle \partial_t^m T_\varphi,\psi \rangle =(-1)^m\int_{\R^d}\rr^*\psi(x)\ \varphi(x) \D x.
  \]
Therefore, ~\eqref{eq:dualityrelation} gives that, for all $\psi\in\cS_0(\Xi)_{\rm even}$,
\begin{align*}
\langle \partial_t^m T_\varphi,\psi \rangle&=(-1)^m\int_{\Xi}\psi(n,t)\ \rr\varphi(n,t) \D n\D t=(-1)^m\langle \rr\varphi,\psi \rangle.
\end{align*}
Therefore, 
$$
\partial_t^m T_\varphi=(-1)^m\rr\varphi\quad\text{in}\quad
\cS_0'(\Xi) ,
$$
and, by \eqref{eq:operatorAdistributions},
$$
T_\varphi=\mathcal{A}^m\partial_t^m
T_\varphi=(-1)^m\mathcal{A}^m(\rr\varphi)\quad
\text{in}\quad \cS_0'(\Xi) .
% \text{in}
% \quad  \begin{cases} \mathcal{S}_0'(\P^d) & \text{if $m$ is even} \\
%   {\mathcal{S}_0'(\Xi)}_{\rm odd} & \text{if $m$ is odd} \end{cases}
% . 
$$
Thus, there exists $p\in\mathcal{P}(\R)$ such that 
$$
T_\varphi=(-1)^m\mathcal{A}^m(\rr\varphi)+p\quad\text{in}\quad \cS'(\Xi).
% \begin{cases} \mathcal{S}'(\P^d) & \text{if $m$ is even} \\  {\mathcal{S}'(\Xi)}_{\rm odd} & \text{if $m$ is odd} \end{cases}.
$$
Hence, 
$$
 T_\varphi(n,t)=(-1)^m\mathcal{A}^m(\rr\varphi)(n,t)+p(t)
$$
for almost every $(n,t)\in\Xi$, and therefore for every $(n,t)\in\Xi$ by continuity. We now show that the polynomial $p$ has to vanish everywhere. Indeed, by the dominated convergence theorem,
\begin{align*}
\lim_{t\to+\infty} |T_\varphi(n,t)|&\leq\lim_{t\to+\infty} \int_{\R^d}|\sigma_m( n\cdot x- t)||\varphi(x)|\D x\\
&=\lim_{t\to+\infty} \frac{1}{(m-1)!}\int_{ n\cdot x\geq t}( n\cdot x- t)^{m-1}|\varphi(x)|\D x\\
&\leq\lim_{t\to+\infty} \frac{1}{(m-1)!}\int_{ n\cdot x\geq t} |x|^{m-1}|\varphi(x)|\D x=0.
\end{align*}
Furthermore, $t\mapsto\mathcal{A}^m(\rr\varphi)(n,t)\in \mathcal{S}_0(\R)$, and thus $\lim_{t\to+\infty}\mathcal{A}^m(\rr\varphi)(n,t)=0$. Hence, we can conclude that $p=0$ and
$$
 T_\varphi(n,t)=(-1)^m\mathcal{A}^m(\rr\varphi)(n,t)
$$
for every $(n,t)\in\Xi$. Observing that
\[
{}_{\cS_0'(\R^d)}\langle\rho_m( \cdot, n ,t), \varphi
\rangle_{\cS_0(\R^d)}= \beta(n,t) T_\varphi(n,t) ,
\]
the claim follows. 
\end{proof}

The space $\mm(\Xi)$
is clearly a subspace of $\cS'(\Xi)$.  
The following simple lemma shows that it is  a subspace of
$\cS_0'(\Xi)$.
\begin{lem} \label{embedding}
Let  $ \mu , \mu' \in \mm(\Xi)$ %(respectively $ \mm(\Xi)_{\rm even} $, $\mm(\Xi)_{\rm odd} $)
be such that $ \mu=\mu' $ in $ \cS_0'(\Xi)$,
then $\mu=\mu'$ in $\mm(\Xi)$. %(respectively $\mm(\Xi)_{\rm even}$, $ \mm(\Xi)_{\rm odd} $).
\end{lem}

%\begin{lem}\label{embedding}
%Let  $ \mu , \mu' \in \mm(\Xi)$ be such that $ \mu=\mu' $
%in $ \cS_0'(\Xi)$,  then $\mu=\mu'$ in $\mm(\Xi)$.
%\end{lem}
\begin{proof}
  Since  $ \cS_0'(\Xi) \simeq \cS'(\Xi) / \pp(\R) $ (see  \Cref{sec:radon}),
  the equality $\mu=\mu'$ in $\cS_0'(\Xi)$ means there exists a
  polynomial $ p \in \pp(\R) $ such that $ \mu' = \mu + p $ in
  $\cS'(\Xi)$. But $p$ must be $0$ since $ \mu, \mu' $ are
  finite measures. Hence, $ \mu' = \mu$ in $\cS'(\Xi)$ and, a
  fortiori, in $\mm(\Xi)$.
 % The even and odd cases are analogous.
\end{proof}
 The next result shows that $\nor{\cdot}{\TV}$ is invariant under symmetrization. 

 \begin{lem}\label{lem:even-odd}
 Let $\mu\in\mm(\Xi)$. Then   
 \[
 \nor{\mu^\vee}{\TV}=\nor{\mu}{\TV}.
 \]
 \end{lem}
 \begin{proof}
 Fix $\mu \in \mm(\Xi)$.
 By definition of $\mu^\vee$ and $\psi^\vee$, 
     \begin{equation}
 \int_{\Xi} \psi(n,t) \ \D\mu^\vee(n,t) =\int_{\Xi} \psi^\vee(n,t) \
 \D\mu(n,t).\label{eq:41}
 \end{equation}
 Indeed, using the above equality and
 $\|\psi^\vee\|_\infty=\|\psi\|_\infty$ for $\psi \in
 \operatorname{C}_0(\Xi)$,  we have
 \begin{align*}
   \nor{\mu^\vee}{\TV} & = \sup \{\scal{\mu^\vee}{\psi}{}{} \colon \psi\in\operatorname{C}_0(\Xi),\|\psi\|_\infty\leq 1 \} \\
   & = \sup \{\scal{\mu}{\psi^\vee}{}{} \colon \psi\in\operatorname{C}_0(\Xi) , \|\psi\|_\infty\leq 1\}  \\
   & = \sup \{\scal{\mu}{\psi}{}{} \colon \psi\in\operatorname{C}_0(\Xi) , \|\psi\|_\infty\leq 1\} =\nor{\mu}{\TV}. \qedhere
 \end{align*}
%  Hence, by triangle inequality,
%   \begin{align*}
%      \nor{\mu}{\TV} & =\nor{\mu_{\rm even}+\mu_{\rm odd}}{\TV} \leq \nor{\mu_{\rm
%          even}}{\TV}+\nor{\mu_{\rm odd}}{\TV} \\
%          & =\nor{\mu}{\TVS}
%          \leq \nor{\mu}{\TV} + \nor{\mu^\vee}{\TV} = 2\nor{\mu}{\TV} . \qedhere
%   \end{align*}
 \end{proof}

% Concerning the feature map $\phi_m$, we further assume that $\beta$
% is an even function if $m$ is even, and an odd function if $m$ is
% odd.  Finally, we ask that $\beta(n,t)\ne0$ for every $(n,t)\in\Xi$.
Equation \eqref{eq:17} shows that the functions $f\in \bb_m$ are
parametrized by the measures $\mu\in \mathcal M(\Xi)$. We now show
that  the even component of $\mu$ can by recovered by the Radon
trasform of $f$.
We recall that $\Lambda^{d-1}$ is the Fourier multiplier defined
by \eqref{eq:27} and \eqref{eq:La-ext}.
\begin{lem} \label{lem:green}
  For every $ f_\mu \in \bb_m$,
  \begin{equation}
    \label{eq:green}
   \frac{1}{2(2\pi)^{d-1}} \partial_t^{m}
 \Lambda^{d-1}\mathcal{R} f_\mu = \beta\ \frac{\mu+(-1)^m\mu^\vee }{2} ,
  \end{equation}
  where the equality holds in $\mathcal{S}_0'(\Xi)$.
\end{lem}
\begin{rmk} \label{rmk:parity}
Observe that $\Lambda^{d-1}\mathcal{R} f_\mu $ is an even distribution
on $\Xi$. Furthermore, it is easy to check that
  \begin{equation*}
\partial_t^m \mathcal{S}_0'(\Xi)_{\rm even} \subseteq
\begin{cases}
  \mathcal{S}_0'(\Xi)_{\rm even} & \text{if $m$ is even} \\
  \cS_0'(\Xi)_{\rm odd} & \text{if $m$ is odd}
\end{cases} . %\label{eq:49}
\end{equation*}
By~\eqref{eq:23} $\beta$ is even, so that $\beta\
(\mu+(-1)^m\mu^\vee )/2$ has the right parity.
 Without condition~\eqref{eq:23}, the statement of~\Cref{lem:green}
 holds true provided that the right hand side of~\eqref{eq:green} is
 replaced with   $(\beta\mu+(-1)^m \beta^\vee\mu^\vee )/2$, which would make
 the decomposition of~\eqref{eq:35} more involved. \end{rmk}
 
\begin{proof}
  Assume first that $m$ is even. As observed in \Cref{rmk:parity}, both sides
  of~\eqref{eq:green}  are even distributions. Thus, it is enough to check
  the equality on $\psi\in \mathcal{S}_0(\Xi)_{\rm even}$. We have
 \begin{align*}
  {}_{\cS_0'(\Xi)}\langle   \partial_t^{m} \Lambda^{d-1}\mathcal{R} f_\mu , \psi\rangle_{\cS_0(\Xi)}
  & = (-1)^m{}_{\cS_0'(\R^d)}\langle f_\mu ,  \mathcal{R}^*\Lambda^{d-1}\partial_t^{m} \psi \rangle_{\cS_0(\R^d)} \\
  & = (-1)^m\int_{\R^d} f_\mu(x) \ \mathcal{R}^*\Lambda^{d-1}\partial_t^{m} \psi(x) \ \D x \\
  & = (-1)^m\int_{\R^d}\left(\int_{\Xi} \rho_m(x,n,t) \ \D\mu(n,t) \right) \mathcal{R}^*\Lambda^{d-1}\partial_t^{m}  \psi(x) \ \D x \\
  & = (-1)^m\int_{\Xi} \int_{\R^d} \rho_m(x,n,t) \ \mathcal{R}^*\Lambda^{d-1}\partial_t^{m}  \psi(x) \ \D x \ \D\mu(n,t) \\
  %& = \int_{\Xi} {}_{\cS_0'(\R^d)}\langle \partial_t^{m} \Lambda^{d-1}\mathcal{R} \phi_m(\cdot)_{\rm odd}(n,t),\varphi \rangle_{\cS_0(\R^d)} \ \D\mu_{\rm odd}(n,t)\\
    & = (-1)^m\int_{\Xi} \langle \rho_m(\cdot,n,t), \mathcal{R}^*\Lambda^{d-1}\partial_t^{m} \psi \rangle \ \D\mu
    (n,t)
  %   & = (-1)^m\int_{\P^d} \langle\sigma_m( n(\xi) \cdot \bullet - t(\xi)), \mathcal{R}^*\Lambda^{d-1}\partial_t^{m} \varphi \rangle\ \beta(\xi)\D\mu
  %  (\xi)
  .
\end{align*}
\Cref{prop:propReLUdirac}, the  inversion
formula~\eqref{bagkprojectionformulaclassical2}
and~\eqref{eq:operatorA} give that, for every $(n,t)\in\Xi$,
 \begin{align*}
  \langle \rho_m(\cdot,n,t), \mathcal{R}^*\Lambda^{d-1}\partial_t^{m}
   \psi \rangle
% &=\beta(n,t)\langle\sigma_m( n \cdot \bullet - t), \mathcal{R}^*\Lambda^{d-1}\partial_t^{m} \psi \rangle\\
&=(-1)^m \beta(n,t) \mathcal{A}^m\mathcal{R}\mathcal{R}^*\Lambda^{d-1}\partial_t^{m} \psi \\
&=(-1)^m 2(2\pi)^{d-1} \beta(n,t)\mathcal{A}^m\partial_t^{m} \psi \\
&=(-1)^m  2(2\pi)^{d-1} \beta(n,t) \psi(n,t) .
%&=(-1)^m \beta(n(\xi),t(\xi))\varphi(n(\xi),t(\xi)).
 \end{align*} 
Thus, taking into account that both $\beta$ (see~\eqref{eq:23}) and $\psi$ are even
functions, we obtain
\begin{align*}
 {}_{\cS_0'(\Xi)}\langle   \partial_t^{m}
  \Lambda^{d-1}\mathcal{R} f_\mu , \psi \rangle_{\cS_0(\Xi)} 
  & = 2(2\pi)^{d-1}  \int_{\Xi} \beta(n,t)\psi(n,t) \ \D\mu(n,t) \\
    & = 2(2\pi)^{d-1} \int_{\Xi} \beta(n,t)\psi(n,t) \ \D\mu_{\rm even}(n,t) \\
 & =2(2\pi)^{d-1}  {}_{\cS_0'(\Xi)}\langle \beta\ \mu_{\rm even}, \psi \rangle_{\cS_0(\Xi)},
 \end{align*}
 which proves~\eqref{eq:green} for even $m$.
 If $m$ is odd, the proof is very similar, observing that both sides
  of~\eqref{eq:green}  are odd distributions, and thus checking
  the equality on $\psi\in \cS_0(\Xi)_{\rm odd}$.
  Furthermore,  $\partial_t^{m}
  \psi$ is an even function, so that $\partial_t^{m}
  \psi\in\cS_0(\P^d)$,  and $\beta\psi$ is an odd function, so that 
  \[
\int_{\Xi} \beta(n,t)\psi(n,t) \ \D\mu(n,t)=\int_{\Xi}
\beta(n,t)\psi(n,t) \ \D\mu_{\rm odd}(n,t) . \qedhere
    \]
\end{proof}
The map $\mu\mapsto f_\mu$ is not injective and next result
characterizes its kernel.
\begin{lem}\label{support}
Let $\mu\in  \mm(\Xi)$. Then:
  \begin{enumerate}[label=\textnormal{(\roman*)}] 
  \item\label{item:1} if $f _\mu=0$,  then
    \[
      \mu^\vee = (-1)^{m+1} \mu \qquad \Longleftrightarrow\qquad \mu\in
      \begin{cases}
        \cS_0'(\Xi)_{\rm odd}  & \text{ if $m$ is even} \\
        \cS_0'(\Xi)_{\rm even}  & \text{ if $m$ is odd} 
      \end{cases};
\]
% \em i.e} $\mu$ is odd or even depending on the fact that $m$ is even
% or odd, respectively;
\item\label{item:2} if  $\mu^\vee = (-1)^{m+1} \mu$,  then $f_\mu$ is a polynomial
  of degree less than $m$. 
\end{enumerate}
Furthermore,  
\[
\mathcal P_m=\{p:\R^d\to\R\colon p \text{ is a polynomial of degree at most }m-1\} ,
\]
where $ \mathcal P_m $  is the space defined in \Cref{main2}.
\end{lem}
\begin{proof}
  Let $\tau=(\mu+(-1)^m \mu^\vee)/2$.  If $f_\mu=0$,
  then~\eqref{eq:green} implies that $\beta \tau=0$ in $\cS_0'(\Xi)$
  and,  by~\eqref{eq:30},  $\tau=0$ in $\cS_0'(\Xi)$ and, by
  \Cref{embedding}, $\tau=0$ in $\mm(\Xi)$.
  
  Assume that $\tau=0$.  Then~\eqref{eq:green} gives that 
$$
\partial_t^{m} \Lambda^{d-1}\mathcal{R} f_\mu = 0
$$
in  $\cS_0'(\Xi)$. Equation~\eqref{eq:operatorAdistributions} implies
that $\partial_t^{m}$ is injective, so that $\Lambda^{d-1}\mathcal{R} f_\mu
= 0$   in  $\cS_0'(\Xi)$. By construction
$\Lambda^{d-1}\mathcal{R} f_\mu\in \cS_0'(\Xi)_{\rm even}$. 
Then, by~\Cref{cor:inversionformuladistributions}, we have that
\begin{equation*}
f_\mu=\frac{1}{2(2\pi)^{d-1}}\rr^*\Lambda^{d-1}\rr f_\mu=0 \quad \text{in\quad $\cS_0'(\R^d) $},
\end{equation*}
or equivalently, there exists $p\in\mathcal{P}(\R)$ such that $f_\mu=p$ in $\cS'(\R^d) $. Hence, 
$$
f_\mu(x)=p(x)
$$
for almost every $x\in\R^d$, and thus for every $x\in\R^d$ by
continuity. But since the elements of $\bb_m$
are functions of at most $m-1$ polynomial growth (see \eqref{eq:19}),
we obtain that $f_\mu$ is a polynomial of degree less than $m$.
We now  prove the last claim.

By \cref{item:2}, $\mathcal   P_m$
is a subspace of the finite-dimensional vector space
of polynomials of degree smaller than $m$. 
Now, let $ \nu = (\delta_{(n,t)} + (-1)^{m+1}  \delta_{(-n,-t)} ) / 2 $ with $(n,t)\in\Xi$. Then, by \eqref{eq:17}  and \eqref{eq:23},
   \begin{align*}
   f_\nu (x) =
\int_{\Xi} \si_m(n' \cdot x - t') \beta(n',t') \ \D\nu(n',t')
%& = \frac{\rho_m(x,n,t)+\red{(-1)^{m+1}} \rho_m(x,-n,-t)}{2}\\
     & = \beta(n,t) \frac{ (n\cdot x-t)^{m-1}}{2(m-1)!},
   \end{align*}
where in the last equality we used
\[
\max\{0,t\}^{m-1}+(-1)^{m+1} \max\{0,-t\}^{m-1}=t^{m-1}.
\]
Then 
\[
\operatorname{span}\{(n\cdot x-t)^{m-1}\colon (n,t)\in \Xi\} \subseteq \mathcal P_m.
\]
However, it is known that the left hand side of the above inequality is the space of polynomials of degree less or equal $m-1$, so that the claim is proved.
\end{proof} 
We are now ready to prove \Cref{main2} and \Cref{without}.
\begin{proof}[Proof of~\Cref{main2}]
We prove the statements for an even $m$ (if $m$ is odd the proof is similar).
We regard $\mathcal Q_m $ and $\mathcal P_m$  as reproducing kernel Banach spaces with the norms
 \begin{subequations}
 \begin{align}
%\mathcal Q_m & = \{ f_\tau \colon\tau\in\mm(\Xi) , \ \tau^\vee=(-1)^m\tau \} ,\\
 \| f \|_{\mathcal Q_m} & = \inf \{ \| \mu \|_{\TV} : \mu \in \mm(\Xi),  \mu^\vee=(-1)^m\mu, f
                     = f_\mu \}, \label{eq:1001} \\
 %\mathcal P_m & = \{ f_\nu \colon\nu\in\mm(\Xi) , \
 %   \nu^\vee=(-1)^{m+1}\nu\} ,
    \| f \|_{\mathcal P_m} & = \inf \{ \| \mu \|_{\TV} : \mu \in \mm(\Xi),  \mu^\vee=(-1)^{m+1}\mu, f
                     = f_\mu \}.\label{eq:1002} 
  \end{align}  
 \end{subequations}
Note that in principle these norms induce respectively on $\mathcal Q_m$ and $\mathcal P_m$ a finer topology than the one induced by the norm~$\nor{\cdot}{\bb_m}$.
Fix $f\in\bb_m$.  By~\eqref{eq:16}, there exists $\mu\in\mm(\Xi)$
such that $f=f_\mu$. Define
\[
\tau= \frac{\mu+\mu^\vee }{2} \in\mm(\Xi)_{\rm
  even}  , \qquad \nu =\frac{\mu- \mu^\vee }{2} \in\mm(\Xi)_{\rm odd},
  \]
and compare with~\eqref{eq:38} taking into account that $m$ is even.
By linearity of the representation~\eqref{eq:17},
%, \Cref{lem:even-odd} implies that
\[
f= f_{\tau} + f_{\nu},
\]
whereas~\cref{item:1} of~\Cref{support} gives
      \begin{equation}
\mathcal Q_m  \cap \mathcal P_m =\{0\}\label{eq:48},
\end{equation}
so that 
\[
\bb_m= \mathcal Q_m + \mathcal P_m,
    \]
and
\begin{equation}\label{eq:1003}
   f_{\tau} =P_{\mathcal Q_m} f,  \qquad  f_\nu= P_{\mathcal P_m} f,
 \end{equation}
 which shows \cref{item:6}.
The fact that $\mathcal   P_m $ is the space of polynomials of degree less or equal $m-1$ is
%\[
%=\operatorname{span}\{(x\cdot   n-t)^{m-1}\colon (n,t)\in \Xi\} 
%  \]
is the content of~\cref{item:2} of~\Cref{support}, whereas  
\cref{item:5} is the content of \cref{item:3} of \Cref{lem:growth}.
Since $\tau$ is the even part of $\mu$, ~\eqref{eq:green} gives
\[
\frac{1}{2 (2\pi)^{d-1}\beta}\partial_t^{m} \Lambda^{d-1}\mathcal{R} f
=\frac{\mu+\mu^{\vee}}{2} =\tau,
  \]
  hence \eqref{eq:24} holds true.

If  $f=f_{\mu'}$ for another $\mu'\in\mm(\Xi)$, by \Cref{support} we have
\[
\mu'= \tau + \nu' , \qquad  \tau=\frac{\mu'+(\mu')^\vee}{2},\qquad f_{\nu'}= f_\nu,
\]
for some  odd measure $\nu'$. 
Taking into account the above equalities, \eqref{eq:normrkbsrelu} gives
  \begin{align}
    \nor{f}{\bb_m} & =\inf \{ \| \tau+\nu'\|_{\TV} : \nu'\in
                     \mm(\Xi)_{\rm odd}  , f_{\nu'}=f_\nu  \} \nonumber\\
& \leq \inf\{ \nor{ \tau }{\TV} +\nor{ \nu'}{\TV}
                     \colon \nu'\in\mm(\Xi)_{\rm odd},
                     f_{\nu'}= f_\nu\} \nonumber\\
    & = \nor{ \tau }{\TV}  +\inf\{  \nor{ \nu'}{\TV}\colon 
                      \nu'\in\mm(\Xi)_{\rm odd}\nonumber,
      f_{\nu'}=f_\nu\} \\
& = \nor{f_\tau}{\mathcal Q_m} +  \nor{f_\nu}{\mathcal P_m} \label{eq:50},
  \end{align}
where the second inequality is a consequence of the triangular inequality, the third one is due to the fact that $\tau$ is even and $\nu'$ is odd, 
and the last equality is a consequence of~\eqref{eq:1001} and~\eqref{eq:1002} observing that 
$\tau$ is the unique even measure such that $f_\tau=P_{\mathcal Q_m} f$, so that
\begin{equation}
    \nor{f_\tau}{\mathcal Q_m} =\nor{\tau}{\TV} \label{eq:1005}.
\end{equation}
% We get 
% \[
%  \nor{f}{\bb_m} \leq \nor{f_\tau}{\mathcal Q_m} +  \nor{f_\nu}{\mathcal P_m},
% \]
% which shows the first inequality of~\eqref{eq:33}. 
Furthermore, by Lemma~\ref{lem:even-odd} we have that
\[
\nor{f_\tau}{\mathcal Q_m}\leq \nor{\frac{\mu'+(\mu')^\vee}{2}}{\TV} \leq \nor{\mu'}{\TV}, \qquad \nor{f_{\nu}}{\mathcal P_m}\leq \nor{\frac{\mu'-(\mu')^\vee}{2}}{\TV} \leq \nor{\mu'}{\TV}.
\] 
Therefore, taking the infimum over all measures $\mu'$ such that $f_{\mu'}=f$, we get
 \begin{equation}
 \nor{f_\tau}{\mathcal Q_m}\leq \nor{f}{\bb_m}, \qquad \nor{f_\nu}{\mathcal P_m}\leq \nor{f}{\bb_m} \label{eq:strange},
 \end{equation}
which, together with \eqref{eq:50}, gives 
\begin{equation}\label{eq:inequalityalmostfinal}
\nor{f}{\bb_m}\leq \nor{f_\tau}{\mathcal Q_m} +  \nor{f_\nu}{\mathcal P_m}\leq 2\nor{f}{\bb_m}.
\end{equation}

If $f\in\mathcal Q_m$, then $f=f_\tau$ and by equations~\eqref{eq:inequalityalmostfinal} and~\eqref{eq:strange} we have that
\[
\nor{f}{\bb_m}\leq \nor{f}{\mathcal Q_m}\leq \nor{f}{\bb_m}.
\]
So that, by~\eqref{eq:1005}
\[
\nor{f}{\bb_m}= \nor{f}{\mathcal Q_m} =\nor{\tau}{\TV},
\]
which proves~\eqref{eq:46}. If
$f\in\mathcal P_m$, then $\tau=0$ and, as above, 
\[
\nor{f}{\bb_m} = \nor{f}{\mathcal P_m} =\inf\{  \nor{ \nu}{\TV}\colon \nu\in\mm(\Xi)_{\rm odd}, f_\nu=f\},
  \]
which is~\eqref{eq:47}. Finally, \eqref{eq:46} and \eqref{eq:47} together with \eqref{eq:inequalityalmostfinal} give equation~\eqref{eq:33}. This also implies that
$\mathcal Q_m$
and $\mathcal P_m$ are closed subspaces of
$\bb_m$. 

We finally prove~\cref{item:4}. Fix a distribution $T$ as in the
statement. By assumption~\eqref{eq:100} and~\Cref{embedding}, there exists a unique  even measure $\tau$ such that 
\begin{align*}
 \tau  =\frac{1}{2(2\pi)^{d-1}\beta}\partial_t^{m} \Lambda^{d-1}\mathcal{R} T ,
 %\\
%\
%    p  & \left(\frac{T+\red{(-1)^{m+1}} T^\vee}{2} \right)\in \mathcal %P_m
\end{align*}
hence $f_\tau\in
\mathcal Q_m$.  Equation \eqref{eq:101}
ensures that there exists $\nu\in\mm(\Xi)_{\rm odd}$ such that   
$T-f_\tau=f_\nu$.  Setting $\mu=\tau+\nu$, we get
\[
 T-f_\mu = (T- f_\tau)- f_\nu = 0,
\]
which proves \ref{item:4}.
\end{proof} 

\begin{proof}[Proof of~\Cref{without}]
  Reasoning as in the last part of the previous proof,
  and again assuming that $m$ is even,
  \eqref{eq:green} implies that
\[
\partial_t^{m} \Lambda^{d-1}\mathcal{R} (T - f_\tau)=0
\]
in $\cS_0'(\Xi)_{\rm even}$. The injectivity of the operator $\partial_t^{m}
\Lambda^{d-1}\mathcal{R}$ gives that $(T - f_{\tau})=0$ in $\cS_0'(\R^d)$,
{\it i.e.} there exists a polynomial $p$ such that 
$
T - f_\tau=p
$
in $\cS'(\R^d)$. 
\end{proof}

\appendix

\section{Sparse solutions in variational problems}\label{sec:bredies-carioni} \label{appendix}
In this section we collect some results from~\cite{MR4040623} that we use
in our paper. 
We start recalling the definition of extremal point.
\begin{dfn}\label{dfn:extr-points-theor}
Let $Q$ be a convex subset of a locally convex space.
A point $ q \in Q $ is called extremal if $ Q \setminus \{ q \} $ is convex.
We denote the set of extremal points of $Q$ by $ \Ext ( Q )$.
\end{dfn}
While extremal points are difficult to characterize in general,
the following result is fairly standard (see \cite[Proposition
4.1]{MR4040623}).   We report the proof for the reader's convenience.
\begin{lem} \label{lem:Ext(B)}
Let $\Th$ be a (Hausdorff) locally compact second countable topological space,
and let
$$
B = \{\mu\in \mm(\Th): \|\mu\|_{\TV} \le 1\}
$$
be the unit ball in $\mm(\Th)$ associated with the total variation norm.
Then
$$ 
\Ext ( B ) = \{ \pm \de_{\th} : \th \in \Th \} .
$$
\end{lem}
\begin{proof}
We start showing that $ \{ \pm \de_{\th} : \th \in \Th \}\subseteq\Ext ( B )$. Let $\th\in\Th$ and $\alpha\in\{-1,1\}$. We suppose that there exist $t\in(0,1)$, $\mu_1,\mu_2\in B$ such that 
\begin{equation}\label{eq:combinationDiracdelta}
\alpha\de_{\th}=t \mu_1+(1-t) \mu_2,
\end{equation}
and we want to show that necessarily $\alpha\de_{\th}=\mu_1=\mu_2$. We observe that the total variation measures $|\mu_1|$, $|\mu_2|$ are probability measures. Indeed, if we suppose on the contrary that $\|\mu_1\|_{\TV},\|\mu_2\|_{\TV}<1$, then 
\[
\|\alpha\de_{\th}\|_{\TV}\leq t \|\mu_1\|_{\TV}+(1-t) \|\mu_2\|_{\TV}<1,
\]
which yields a contradiction. Furthermore, 
\[
\de_{\th}=t |\mu_1|+(1-t) |\mu_2|.
\]
Indeed, we  first  observe that $(t |\mu_1|+(1-t) |\mu_2|)(\Th)=1$ and 
\[
\de_{\th}=|\de_{\th}|\leq t |\mu_1|+(1-t) |\mu_2|.
\]
Then, for every Borel set $E\subseteq\Th$, if $\th\in E$
$$
1=\de_{\th}(E)\leq (t |\mu_1|+(1-t) |\mu_2|)(E)\leq 1,
$$
and if $\th\in \Th\setminus E$
$$
(t |\mu_1|+(1-t) |\mu_2|)(E)=(t |\mu_1|+(1-t) |\mu_2|)(\Th)-(t |\mu_1|+(1-t) |\mu_2|)(\Th\setminus E)=0.
$$
Therefore, $ |\mu_1|= |\mu_2|=\de_{\th}$, which implies $\mu_1=\alpha_1 \de_{\th}$ and $\mu_2=\alpha_2 \de_{\th}$ with $|\alpha_1|=|\alpha_2|=1$, and equation \eqref{eq:combinationDiracdelta} becomes 
\begin{equation}\label{eq:combinationDiracdelta2}
\alpha\de_{\th}=(t \alpha_1 +(1-t) \alpha_2) \de_{\th}.
\end{equation}
Since $\alpha,\alpha_1,\alpha_2\in\{-1,1\}$, equation \eqref{eq:combinationDiracdelta2} is satisfied if and only if $\alpha=\alpha_1=\alpha_2$. So that, $\alpha\de_{\th}=\mu_1=\mu_2$, and then $\alpha\de_{\th}\in \Ext ( B )$. It remains to prove the opposite inclusion $\Ext ( B )\subseteq \{ \pm \de_{\th} : \th \in \Th \}$. We suppose that there exists $\mu\in\mm(\Th)$ such that $\mu\notin\{ \pm \de_{\th} : \th \in \Th \}$ but $\mu\in\Ext ( B )$. Then, $\|\mu\|_{\TV}=1$. We denote by $\chi_E$ the indicator function on a subset $E\subseteq \Th$. For every Borelian set $E$ such that $|\mu|(E)\in(0,1)$, we can rewrite $\mu$ as the linear combination
\[
\mu=\mu\cdot \chi_E + \mu\cdot \chi_{\Th\setminus E}=t\frac{\mu\cdot \chi_E}{|\mu|(E)}+(1-t)\frac{\mu\cdot \chi_{\Th\setminus E}}{|\mu|(\Th\setminus E)},
\]
where $t=|\mu|(E)\in(0,1)$. Since $\mu\notin\{ \pm \de_{\th} : \th \in \Th \}$, then it is possible to find a Borelian set $E$ such that $\mu\ne |\mu|(E)^{-1}\mu\cdot \chi_E$ and $\mu\ne |\mu|(\Th\setminus E)^{-1}\mu\cdot \chi_{\Th\setminus E}$. This shows that there exist $t\in(0,1)$, $\mu_1,\mu_2\in B$ such that $\mu=t \mu_1+(1-t) \mu_2$, which yields a contradiction. Therefore, we have shown that $\Ext ( B_{ \TV }(1) )\subseteq \{ \pm \de_\th : \th \in \Th \}$, which concludes the proof.
\end{proof}

To establish our representer theorem we recall the following  known result.
\begin{thm}[{\cite[Theorem 3.3]{MR4040623}}] \label{thm:bredies-carioni}
Consider the problem
\begin{equation} \label{eq:problem0}
 \inf_{u \in U} F( \aa u ) + G(u) ,
\end{equation}
where
$U$ is a locally convex topological vector space,
$ \aa : U \to H  $ is a continuous, surjective linear map with values in a finite-dimensional Hilbert space $H$,
$ F : H \to (-\infty,+\infty] $
is proper, convex, coercive and lower semi-continuous,
and $ G : U \to [0,+\infty) $ is a coercive and lower semi-continuous norm.
Then \eqref{eq:problem0} has solutions of the form
$ \sum_{i=1}^K \ga_i u_i $
with $ K \le \dim H $, $ \ga_i > 0 $, $ \sum_{i=1}^K \ga_i = G(u)  $,
and $ u_i \in \Ext( \{ u \in U : G(u) \le 1 \} ) $.
\end{thm}

\Cref{thm:bredies-carioni} is a simplified version of \cite[Theorem 3.3]{MR4040623},
where $G$ is only assumed to be a seminorm.  In such a case, the statement needs to take care of the kernel of $G$.
A seminorm $G$ is called coercive if, for all $R>0$, the set
\[
  \{ [u]\in U/\mathcal N : G(u)\leq R\}
\]
is compact in $U/\mathcal N$, where $\mathcal N$ is the kernel of $G$
(see Assumption~[H1] in  \cite{MR4040623}). 
\begin{rmk}\label{rmk:sparse-solution}
  In \Cref{thm:bredies-carioni}, the space $U$ is endowed with a
  topology weaker than the topology induced by the norm $G$ in order
  to ensure that   the closed balls are compact.
\end{rmk}

%\bibliographystyle{plain}
%\bibliography{biblio}

%\bibliographystyle{plain}
%\bibliography{biblio}

\section*{Acknowledgements}
The authors would like to thank Filippo De Mari, and especially Jaouad Mourtada and Koen Sanders for useful discussions. Indeed, the ideas in this paper were initially explored as Koen's summer project under the co-supervision of Jaouad. L.R. acknowledges support from the Center for Brains, Minds and Machines (CBMM), funded by NSF STC award CCF-1231216. L.R. also acknowledges the financial support of the European Research Council (grant SLING 819789), the AFOSR projects FA9550- 18-1-7009, FA9550- 17-1-0390 and BAA-AFRL-AFOSR-2016-0007 (European Office of Aerospace Research and Development), and the EU H2020-MSCA-RISE project NoMADS - DLV-777826. E.D.V. is a member of the Gruppo Nazionale per l’Analisi Matematica, la Probabilit\`a e le loro Applicazioni (GNAMPA) of the Istituto Nazionale di Alta Matematica (INdAM).

\printbibliography

\end{document}